\documentclass{article}
\usepackage{hyperref}
\usepackage{graphicx} 
\usepackage{amsthm}
\usepackage{amsmath}
\usepackage{amssymb}
\usepackage{enumitem}
\usepackage{color}
\usepackage{subcaption}
\usepackage{booktabs}
\usepackage{pifont}

\usepackage[top=3cm, bottom=2cm, left=3cm, right=2cm]{geometry} 

\def\U{\mathcal U}
\def\LL{\mathcal L}
\def\P{\mathcal P}
\def\D{\mathcal D}

\usepackage{algorithmic}
\usepackage[ruled,vlined,linesnumbered]{algorithm2e}
\usepackage{xcolor}
\usepackage{multicol} 
\usepackage{enumitem}

\SetCommentSty{mycommfont}

\newcommand{\rd}[1]{{\color{red} #1}}

\newtheorem{theorem}{Theorem}
\newtheorem{assumption}[theorem]{Assumption}
\newtheorem{lemma}[theorem]{Lemma}

\title{Verifying   the Selected Completely at Random Assumption in Positive-Unlabeled Learning}
\author{Paweł Teisseyre, Konrad Furmańczyk, Jan Mielniczuk}
\date{}

\begin{document}

\maketitle

\abstract{
The goal of positive-unlabeled (PU) learning is to train a binary classifier on the basis of training data containing positive and unlabeled instances, where unlabeled observations can belong either to the positive class or to the negative class. Modeling PU data requires certain assumptions on the labeling mechanism that describes which positive observations are assigned a label. The simplest assumption, considered in early works, is SCAR (Selected Completely at Random Assumption), according to which the propensity score function, defined as the probability of assigning a label to a positive observation, is constant. On the other hand, a much more realistic assumption is SAR (Selected at Random), which states that the propensity function solely depends on the observed feature vector. SCAR-based algorithms are much simpler and computationally much faster  compared to SAR-based algorithms, which usually require challenging estimation of the propensity score. In this work, we propose a relatively simple and computationally fast test that can be used to determine whether the observed data meet the SCAR assumption. Our test is based on generating artificial labels conforming to the SCAR case, which in turn allows  to mimic the distribution of the test statistic under the null hypothesis of SCAR.  We justify our method theoretically.
In experiments, we demonstrate that the test successfully detects various deviations from SCAR scenario and at the same time it is possible to effectively control the type I error. The proposed test can be recommended as a pre-processing step to decide which final PU algorithm to choose in cases when nature of labeling mechanism is not known.
}

\section{Introduction}

Learning from positive-unlabeled data (PU learning) is an active research topic that has attracted great deal of  interest in the machine learning community in recent years \cite{BekkerDavis2020, ElkanNoto2008, Gong2022}.
The goal of PU learning is to train a binary classifier on the basis of training data containing positive and unlabeled instances, where unlabeled observations can belong either to the positive or to the negative class. 
The problem is motivated by many practical applications.
A representative example is detection of illegal or harmful content in social networks. Some profiles
are reported as containing such content (positive cases). However, profiles not reported as illegal may also contain content that violates the law, but this has not been verified. 
Another example  is reporting side effects of taking medications. The lack of a reported side effect does not mean that it did not occur. Therefore, it is reasonable to treat reported cases as positive and unreported cases as  unlabeled.
PU data appear naturally in the classification of texts and
images \cite{Chiaroni2018}, anomaly detection \cite{Luo2018, Perini2023}, survey research \cite{Sechidis2017} and in many bioinformatics applications \cite{Li2021}.

The simplest approach in PU learning (called naive or biased method) is to treat all unlabeled observations as negative and use standard binary classifiers.
However, this method may lead to a significantly  biased posterior probability estimate for the true class variable and consequently to poor classification accuracy, especially if the unlabeled set contains relatively  many positive cases.
Therefore, most authors approach modeling  PU data by imposing certain assumptions on the labeling mechanism that describes which positive observations are  labeled. 

\begin{figure}
\centering
    \begin{tabular}{c c}
      \includegraphics[width=0.3\textwidth]{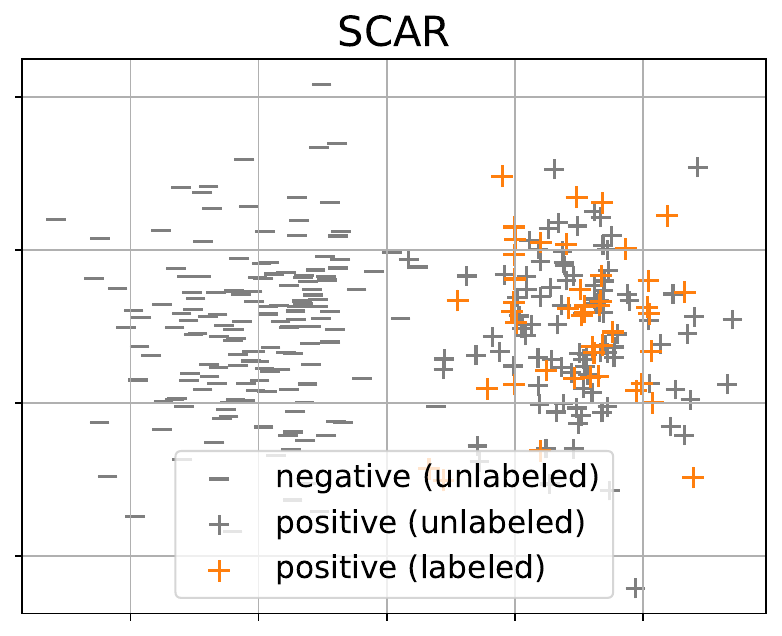} &
      \includegraphics[width=0.3\textwidth]{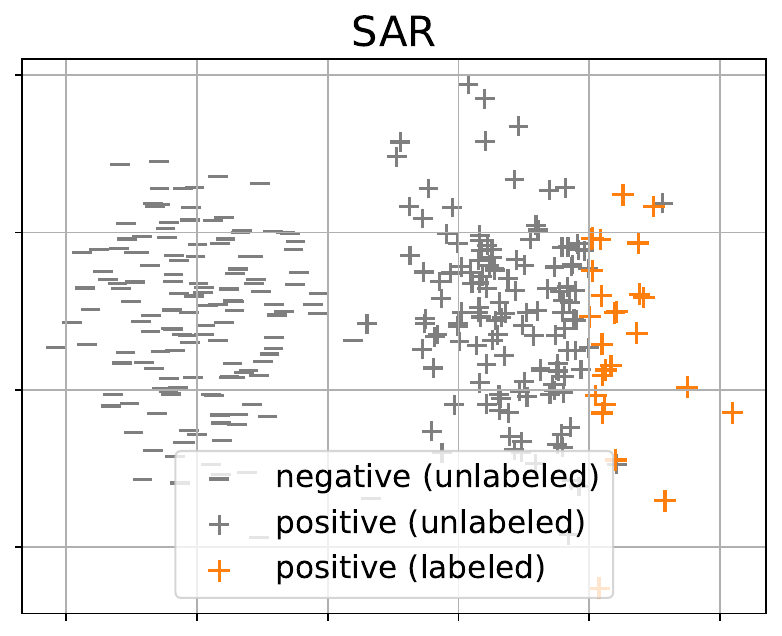}  \\
      \end{tabular}
    \caption{Visualization of SCAR and SAR settings. Under the SCAR assumption, the probability of labeling positive observations does not depend on the feature vector while under SAR this probability depends on the features.}
    \label{fig:scar_sar}
\end{figure}

The simplest assumption is SCAR (Selected Completely at Random Assumption), according to which the propensity score function, i.e.  the probability of a labeling   a positive observation, is constant \cite{ElkanNoto2008,BekkerDavis2020, ChenLiuWangZhaoWu2020, Zhao2022, li2022who}. 
Under the SCAR assumption, a possible approach is to estimate label frequency \cite{Ramaswamy2016, Jainetal2016, BekkerAAAI18, Lazeckaetal2021} and then use it to scale the posterior probabilities obtained from the naive method or, alternatively, optimize weighted empirical risk function with weights depending on the label frequency \cite{BekkerDavis2020, ICCS2020}.  
Generally, SCAR based algorithms are relatively simple and computationally fast. However, the SCAR assumption is not met in many practical situations \cite{Gong2022}.
For example, among people experiencing drug side effects, the likelihood of reporting may depend on age or socioeconomic factors.

A much more realistic assumption is SAR (Selected at Random), which states that the propensity score function depends solely on the observed feature vector \cite{BekkerRobberechtsDavis2019, Gerych2022, Gong2021, Gong2022, FurmanczykECAI2023, Verreet2023}. 
Figure \ref{fig:scar_sar} shows the difference between SCAR and SAR assumptions for artificially generated two-dimensional data.
However SAR based algorithms are usually computationally more expensive as they require challenging estimation of the propensity score. An exception is the situation when we consider assumptions that are special cases of SAR, such as Probabilistic Gap Assumption \cite{Gerych2022}, invariance of order assumption \cite{Kato2019} or
 impose some additional assumptions such  as knowledge of prior probability of positive class \cite{NaVAE}.
Most existing SAR algorithms are based on the   alternating fitting of two models: one is related to the posterior probability of the true class variable, and the other is related to the propensity score \cite{BekkerRobberechtsDavis2019, Gong2021, FurmanczykECAI2023}. For example, SAR-EM \cite{BekkerRobberechtsDavis2019} and LBE \cite{Gong2021} are based on an EM-type algorithm, whereas TM \cite{FurmanczykECAI2023} relies on iterative approximation of a set of positive observations and using this set to estimate the propensity score.

These approaches require many iterations, each of which includes training the classifier. Moreover, importantly,  applying one of these methods, when in reality propensity score is constant, leads to  a loss of  efficiency with respect to  SCAR-designed approaches. Table \ref{tab:SCAR_SAR} contains a comparison of representative methods based on SCAR and SAR, in a situation where the true propensity score is constant and equal to $0.5$. As the SCAR method, we used the popular estimator of labeling frequency $c$  called TiCE \cite{BekkerAAAI18} and then scaled the posterior probabilities obtained from the naive model by $c^{-1}$. As the SAR method, we used the LBE algorithm \cite{Gong2021} mentioned above. The SCAR-based method has higher classification accuracy for most  considered datasets and, importantly, significantly shorter training time.
Therefore, verifying the SCAR assumption becomes an important task that, to our knowledge, has not been discussed in the literature.

\begin{table}
\caption{Comparison of classification accuracy and training time for typical SCAR \cite{BekkerAAAI18} and SAR methods \cite{Gong2021} under SCAR setting. For SCAR method we use TICE algorithm \cite{BekkerAAAI18} and scale the outpunt of the naive classifer, for SAR we used LBE method \cite{Gong2021}.}
\label{tab:SCAR_SAR}
\centering
\begin{tabular}{l||ll|ll}
\toprule 
 & \multicolumn{2}{c}{SCAR method} & \multicolumn{2}{c}{SAR method}  \\
Dataset & Accuracy	& Time & Accuracy& Time  \\
\midrule
Breast &\textbf{0.969 $\pm$ 0.017} &	1.62 $\pm$ 0.02&	0.954 $\pm$ 0.011&	18.4 $\pm$ 0.80 \\
Wdbc&\textbf{0.930 $\pm$ 0.022}&	2.94 $\pm$ 0.04&	\textbf{0.930 $\pm$ 0.019}&	19.6 $\pm$ 0.80\\
Banknote&0.983 $\pm$ 0.004&	1.49 $\pm$ 0.01&	\textbf{0.987 $\pm$ 0.012}&	20.8 $\pm$ 0.40\\
Segment&0.972 $\pm$ 0.004&	8.69 $\pm$ 0.15&	\textbf{0.990 $\pm$ 0.006}&	24.8 $\pm$ 1.91\\
CIFAR10$^*$ &\textbf{0.810 $\pm$ 0.009}&	5.36 $\pm$ 0.11&	0.718 $\pm$ 0.031&	25.8 $\pm$ 0.74\\
USPS$^*$&\textbf{0.726 $\pm$ 0.022}&	5.47 $\pm$ 0.21&	0.712 $\pm$ 0.017&	26.4 $\pm$ 0.80\\
Fashion$^*$&\textbf{0.824 $\pm$ 0.014}&	5.28 $\pm$ 0.15&	0.816 $\pm$ 0.021&	25.4 $\pm$ 0.49\\
\bottomrule
\end{tabular}
\\
$^*$  Randomly chosen subsamples of 5000 images were considered.
\end{table}

In this work, we propose a relatively simple and computationally fast test that can be used to determine whether the observed data meet the SCAR assumption. 
The proposed procedure consists of two steps. In the first step, our goal is to determine the set of positive observations. In the second step, we generate artificial labels conforming to the SCAR situation, which in turn allows us to mimic the distribution of the test statistic under the null hypothesis of SCAR.
The idea of the method is based on the property that the SCAR assumption is equivalent to the equality of the distribution of the feature vector for positive observations and the distribution for labeled observations. This leads to the selection of 4 different test statistics that measure the divergence between the above distributions.
In experiments, we demonstrate that the test successfully detects various SAR schemes and at the same time it is possible to effectively control type I error (observed significance level) for most considered datasets. This is supported by theoretical results  which show that (i)  the proposed test is indeed consistent and that (ii) the essential part of the proposal, namely selection of positive elements among unlabeled ones satisfies probabilistic guarantees in an idealized scenario.
The proposed test can be recommended as a pre-processing step to decide which final PU algorithm to choose.

\section{Background}

\subsection{Positive-unlabeled learning}
In PU learning, each observation can be described by the triple $(X,S,Y)$, where $X \in R^{d}$ is feature vector, $Y \in \{0,1\}$ is true class variable ($Y=1$ denotes positive class), which is not observed directly and  $S \in \{0,1\}$ is class label indicator, describing whether the instance is labeled and thus positive ($S=1$) or unlabeled $(S=0)$. The unlabeled instance can be either positive or negative. In PU learning it is assumed that negative examples cannot be labeled, i.e., $P(S=1|Y=0)=0$.
The fraction of positive observations that are assigned a label is determined by the labeling frequency $c=P(S=1|Y=1)$.
In this work, we adopt a single-training sample scenario \cite{BekkerDavis2020} assuming that iid random vectors $(X_i,Y_i,S_i)$ for $i=1,\ldots,n$ are generated from some unknown distribution $P_{X,Y,S}$. The PU training data is $\mathcal{D}=\{(X_i,S_i):i=1\ldots,n\}$ as we do not observe $Y_i$. The goal is to train a classifier that predicts $Y$ for some new instance $X$ using the incompletely labeled training set $\mathcal{D}$ only.
Note that training the naive classifier which treats $S$ as the class variable, we can estimate $s(x)=P(S=1|X=x)$, whereas our goal is to estimate $y(x)=P(Y=1|X=x)$.
Table \ref{Tab1} contains the most important notations used in the paper.

\begin{table}
{\small
\begin{center}
\caption{Summary of notation.}
\label{Tab1}
\begin{tabular}{ll }
\hline
Notation & Meaning \\
\hline
$n$ & number of instances \\
$d$ & number of features\\
$X\in R^d$ & feature vector \\
$Y\in\{0,1\}$ & unobserved true class variable \\
$S\in\{0,1\}$ & label indicator \\
$\D=\{(X_i,S_i):i=1\ldots,n\}$ & PU training data \\
$\pi=P(Y=1)$ & class prior \\
$c=P(S=1|Y=1)$ &  labeling frequency\\
$\P=\{i:Y_i=1\}$ & positive set (unobserved) \\
$\LL=\{i:S_i=1\}$ & labeled set (observed) \\ 
$\U=\{i:S_i=0\}$ & unlabeled set (observed) \\ 
$y(x)=P(Y=1|X=x)$ & posterior probability of $Y=1$ \\
$s(x)=P(S=1|X=x)$ & posterior probability of $S=1$ \\
$e(x)=P(S=1|X=x,Y=1)$ & propensity score function\\
$c=P(S=1|Y=1)$ &  labeling frequency\\
\hline
\end{tabular}
\end{center}
}
\end{table}

\subsection{SCAR and SAR assumptions}
Learning from PU data is challenging task and certain assumptions are required to make inference from PU data possible. The assumptions concern the labeling mechanism describing which positive observations are assigned a label.
Specifically, the labeling mechanism assigns a probability $e(x)=P(S=1|X=x,Y=1)$, called propensity score, of being labeled to each positive example.
A high value of propensity score indicates that a positive observation described by vector $x$ will be assigned a label with a high probability. The two assumptions which we  want to check can be expressed in terms of the propensity score function.
\begin{assumption}[\textbf{S}elected \textbf{C}ompletely \textbf{a}t \textbf{R}andom]
Propensity score is constant $e(x)=c$.
\end{assumption}
Clearly, SCAR is unlikely to hold in many situations. Therefore, many works consider a more general and less restrictive assumption called SAR.
\begin{assumption}[\textbf{S}elected \textbf{a}t \textbf{R}andom]
Propensity score is non-constant function $e(x)=P(S=1|X=x,Y=1)$ depending solely on the observed features $x$.
\footnote{Typically, SAR refers to the situation where e(x) can be any function with values in $[0,1]$, which also includes the case of a constant function (SCAR). In this work, it is more convenient to assume that SAR  refers to non-constant propensity score.}
\end{assumption}
Importantly, various different labeling mechanisms fall in SAR category, including cases where the labeling mechanism depends on a single feature or on many of them simultaneously. 
More generally, the labeling mechanism depends on variables that are not observed in our data and then SAR is not met. However, in such a situation, modeling of PU data becomes impossible unless all  active predictors become available. Therefore, SAR can be treated as the most general assumption made in PU learning.

The SCAR assumption can be characterized by the following property \cite{BekkerDavis2020, ElkanNoto2008}; for completeness, we provide proof in the supplement \footnote{\url{https://github.com/teisseyrep/SCAR/tree/main}}. It states that
SCAR is equivalent to 
\begin{equation}
\label{Eq_distr}
P(X=x|S=1)=P(X=x|Y=1),    
\end{equation}
which means that the feature distribution for labeled observations matches the feature distribution in the positive class.
Property (\ref{Eq_distr}) will be used to construct a proposed testing procedure.


\section{Verifying the SCAR assumption in PU learning}

\subsection{Null and alternative hypotheses}
\label{Sec:Testing procedure}
Our goal is to verify the SCAR assumption based on PU data and determine which mechanism corresponds to how our PU data was generated: SCAR or SAR.
We use a statistical hypothesis testing framework. In view of property (\ref{Eq_distr}), the null and alternative hypotheses can be written as
\begin{align*}
    H_0:  P_{X|S=1}&=P_{X|Y=1} \quad (\text{SCAR})  \\
    H_1:  P_{X|S=1}&\neq P_{X|Y=1} \quad (\text{SAR}).  
\end{align*}
The test statistic $T(\widehat{P}_{X|S=1}, \widehat{P}_{X|Y=1})$ should measure how close the empirical distributions $\widehat{P}_{X|S=1}$ and $\widehat{P}_{X|Y=1}$ corresponding to the true distributions $P_{X|S=1}$ and $P_{X|Y=1}$ distributions are. A small value of the test statistic should indicate $H_0$, while large values of the statistic should lead to its rejection.
In Section 3.3, we present possible test statistics that can be used to measure the divergence between these two distributions.
However, even with a defined test statistic, we face two challenges. First, distribution $P_{X|Y=1}$ cannot be directly estimated because we do not observe $Y$. Second, we need to know the distribution of $T$ under $H_0$  to determine which values of $T$ are typical under $H_0$ and consequently be able to control for the type I error.
The above two issues are addressed in section \ref{Sec:Testing procedure}. 

Finally, it is worth noting that one can make two errors: type I error (reject $H_0$ when it is true) and type II error (not reject $H_0$ when $H_1$ is true). The above errors are not symmetric. Rejection of $H_0$ suggests applying SAR-based algorithms, which are usually more demanding computationally but which are also valid in  SCAR situations. 
On the other hand, SCAR algorithms are unable to estimate non-constant propensity functions and thus may fail in some situations related to SAR. Therefore, a type II error can potentially have more serious negative consequences.

\subsection{Testing procedure}
The testing procedure consists of two steps. In step (1) our goal is to approximate the positive set $\P=\{i:Y_i=1\}$, and in step (2) it is to generate the distribution of the test statistic under $H_0$ (null distribution).
In the following, we assume that the class prior $\pi=P(Y=1)$ is known, although in practice it is usually replaced by an estimated value. The assumption is commonly adopted in PU inference \cite{ward2009,uPU,PUSB}. We note in passing that for estimation of $\pi$ under SAR, estimation of   posterior probability seems unavoidable, this however  requires assumptions in its turn to ensure identifiability. 

Step (1) involves training a naive model in which $S$ is treated as a class variable. This allows us to estimate $\widehat{s}(X_i)$, $i\in \U$ and then sort the unlabeled observations in descending order: $\widehat{s}(X_{i_1})\geq\ldots\geq \widehat{s}(X_{i_m})$, where $\U=\{i_1,\ldots,i_m\}$. The positive set $\P$ is estimated as 
 the sum of the labeled set $\LL=\{i:S_i=1\}$ and the set of unlabeled observations with the highest estimated posterior probabilities $\widehat s(x)$, i.e., $\widehat{\P}=\LL\cup \{i_1,\ldots,i_k\}$, where $k=n\pi(1-\widehat{c})$ and $\widehat{c}=\widehat{P}(S=1)/\pi$. Estimator $\widehat{P}(S=1)$ is simply a fraction of labeled exampels in training data. The rationale for this method of estimating $\mathcal{P}$ is explained in Lemma \ref{Lemma_Ordering} in Section \ref{Theory}.
Note that $\widehat{\P}$ contains approximately $n\pi$ observations, which corresponds to the expected number of  observations in $\P$.
We also define variable $\widetilde{Y}_i=1$ iff $i\in\widehat{\P}$, which approximates the true class indicator $Y$.

In step (2), we generate the artificial label indicator $\widetilde{S}$ which mimics a true label indicator $S$, but corresponds to a SCAR situation. Specifically,  for each $i\in\widehat{\P}$ we generate $\widetilde{S}_i\in\{0,1\}$ from Bernoulli distribution with success probability $P(\widetilde{S}_i=1)=\widehat{c}$ and we set $\widetilde{S}_i=0$, for $i\notin\widehat{P}$. The above step is repeated for $b=1,\ldots,B$ and in each loop we compute $T_b:=T(\widehat{P}_{X|\widetilde{S}=1},\widehat{P}_{X|\widetilde{Y}=1})$.   
Finally, based on the values $T_1,\ldots,T_B$, we can estimate the distribution of the $T$ statistic under $H_0$. Figure \ref{fig:Alg1} visualizes steps (1) and (2) for SAR dataset.
The higher the value of parameter $B$, the better the approximation of the distribution under $H_0$, but at the same time the greater the computational cost.

\begin{figure}
\centering
    \begin{tabular}{c c}
      \includegraphics[width=0.3\textwidth]{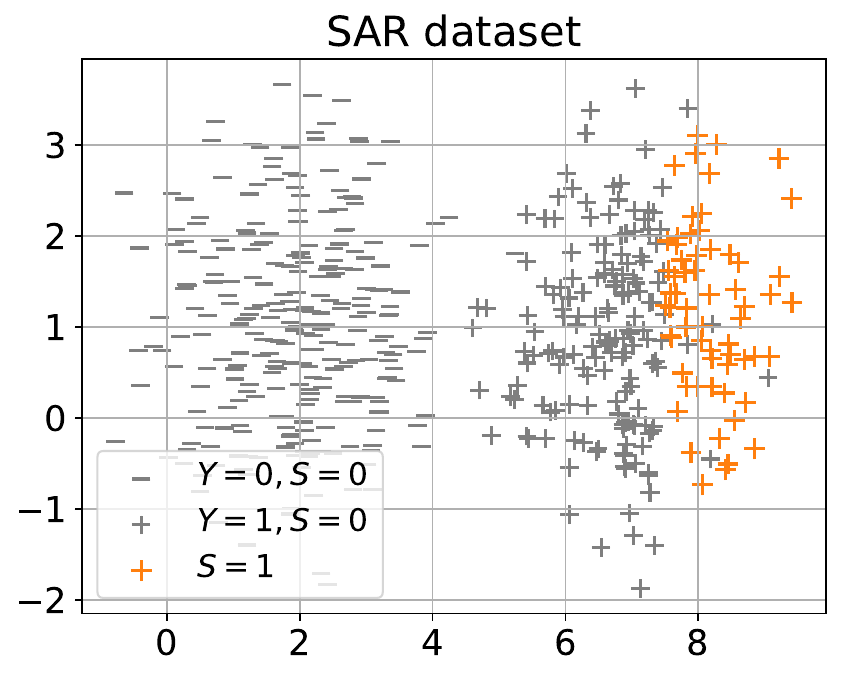} &
      \includegraphics[width=0.3\textwidth]{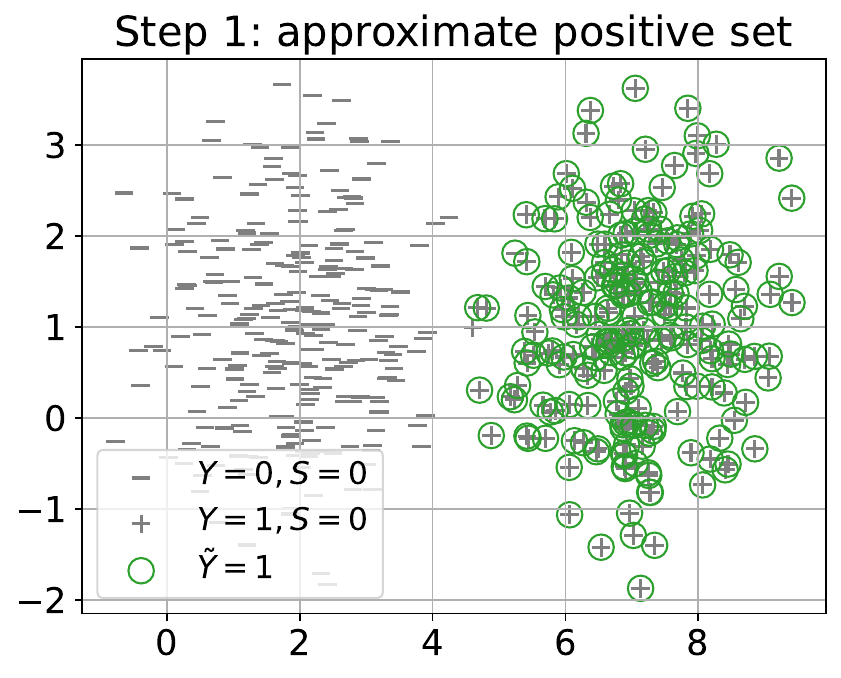}  \\
      \includegraphics[width=0.3\textwidth]{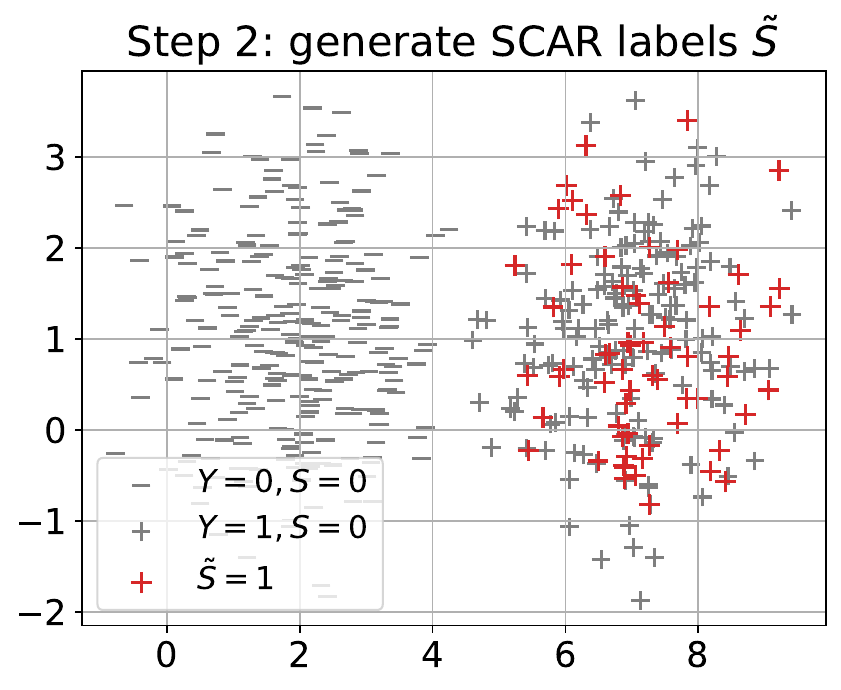} &
      \includegraphics[width=0.3\textwidth]{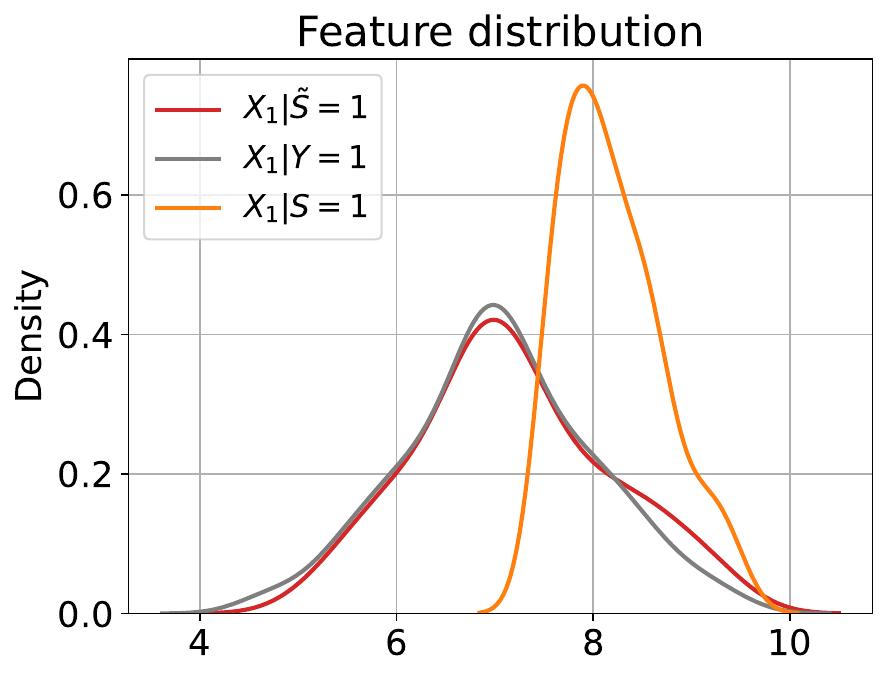}  \\
      \end{tabular}
    \caption{The visualization shows how in Algorithm \ref{Alg1} artificial labels $\widetilde{S}$ matching the SCAR assumption are generated.}
    \label{fig:Alg1}
\end{figure}

The last step is to calculate the p-value $\widehat{p}=\#\{b:T_b\geq T_0\}/B$, where $T_0:=T(\widehat{P}_{X|S=1},\widehat{P}_{X|\widetilde{Y}=1})$ is the value of test statistic for the observed label indicator $S$. A small p-value indicates that $T_0$ takes unusually large values compared to the values corresponding to $H_0$, which should lead to the rejection of $H_0$. Formally, we reject $H_0$, when $\widehat{p}<\alpha$, where $\alpha\in (0,1)$ is user-specified significance level. 
The whole procedure is described by Algorithm \ref{Alg1}.

\begin{algorithm}
\caption{Verifying  SCAR assumption}\label{alg-prop}
\begin{algorithmic}[1]
\STATE Input: PU training data $\mathcal{D}=\{(X_i,S_i):i=1\ldots,n\}$, test statistic $T$, number of repetitions $B$, significance level $\alpha$, class prior $\pi$

\tcc{Approximate positive set:}
\STATE Train the naive classifier using $\D$ and estimate $\widehat{s}(X_i)$, $i\in \U$.
\STATE Sort $\widehat{s}(X_{i_1})\geq\ldots\geq \widehat{s}(X_{i_m})$, where $\U=\{i_1,\ldots,i_m\}$.
\STATE Let $\widehat{\P}=\LL\cup \{i_1,\ldots,i_k\}$, where $k=n\pi(1-\widehat{c})$.
\STATE Define $\widetilde{Y}_i=1$ if $i\in\widehat{\P}$ and $\widetilde{Y}_i=0$ otherwise. 

\CommentSty{\color{blue}}
\tcc{Generate distribution of T under H0:}
\FOR{$b \gets 1$ to $B$}                    
            \FOR{$i\in \widehat{\P}$}
            \STATE Draw $\widetilde{S}_i\in\{0,1\}$, such that $P(\widetilde{S}_i=1)=\widehat{c}$
            \ENDFOR
            \STATE Set $\widetilde{S}_i=0$, for $i\notin\widehat{\P}$
            \STATE Calculate $T_b:=T(\widehat{P}_{X|\widetilde{S}=1},\widehat{P}_{X|\widetilde{Y}=1})$
\ENDFOR

\tcc{Compute p-value:}
\STATE Let $T_0:=T(\widehat{P}_{X|S=1},\widehat{P}_{X|\widetilde{Y}=1})$ 
\STATE Compute p-value $\widehat{p}=\#\{b:T_b\geq T_0\}/B$.
\STATE Output: Reject $H_0$ iff $\text{p-value}<\alpha$.  
\end{algorithmic}
\label{Alg1}
\end{algorithm}


\subsection{Test statistics}
\label{Sec:Test statistics}
Algorithm \ref{Alg1} is generic and allows any statistic to be used. However, for the algorithm to work effectively, the statistics should meet two requirements. Firstly, it should describe the deviation from $H_0$, in the sense that its theoretical value should be $0$ for $H_0$ and take positive values for $H_1$. Second, it should be computationally fast because we have to compute it $B$ times. In this chapter, we present 4 possible statistics that meet the above conditions.

Let us denote by $P_1$ and $P_2$ the probability distributions corresponding to $P_{X|\widetilde{S}=1}$ and $P_{X|\widetilde{Y}=1}$ in Algorithm \ref{Alg1}. A natural way to measure how different the two distributions are is to use the Kullback-Leibner (KL) divergence \cite{Cover2006, Belghazietal2018}. Despite the desirable properties, calculating KL is computationally demanding for multidimensional distributions. However, under certain assumptions, computations can be simplified.
For example, assuming a Gaussian distribution of features, we obtain:
\begin{equation}
\label{Stat_KL}
T(P_1,P_2)=0.5\left[r^{T}\Sigma_{1}^{-1}r + \text{tr}(\Sigma_2^{-1}\Sigma_1)-\log(\frac{|\Sigma_1|}{|\Sigma_2|})-d\right]
\end{equation}
where $r:=\mu_2-\mu_1$ is a difference between the means $\mu_1$ and $\mu_2$ for the $P_1$ and $P_2$ distributions, respectively, whereas $\Sigma_1$ and $\Sigma_2$ are the corresponding covariance matrices. In the experiments, we consider two variants of (\ref{Stat_KL}): the first one is based on the assumption of independence of variables (we simply denote it as KL) and the second one in which we estimate the covariance matrices (denoted as KLCOV). 

In addition, we consider Kolmogorov-Smirnov (KS) statistic defined as
\begin{equation}
\label{Stat_KS}
T(P_1,P_2)=\sum_{j=1}^{d}KS(P_{1,j},P_{2,j}),
\end{equation}
where $P_{1,j}$ and $P_{2,j}$ are marginal distributions corresponding to the multivariate distributions $P_1$ and $P_2$ and $KS(P_{1,j},P_{2,j})$ is standard  Kolmogorov-Smirnov statistic for one-dimensional probability distributions.

Finally, we also consider a classifier-based statistic.
Since we want to decide how much the distributions $\widehat{P}_{X|\widetilde{S}=1}$ and $\widehat{P}_{X|\widetilde{Y}=1}$ differ from each other, we define an auxiliary class variable $Z_i\in\{1,-1\}$ such that $Z_i=1$ if $\widetilde{S}_i=1$ and $Z_i=-1$ if $\widetilde{Y}_i=1$.
Then we train simple Naive Bayes classifier  using training data 
$\mathcal{D}_z=\{(X_i,Z_i)\}$. Other classifiers can also be used as long as their training time is acceptable.
We measure the quality of the classifier using $ROC\, AUC$  and define test statistic as $T(\widehat{P}_{X|\widetilde{S}=1},\widehat{P}_{X|\widetilde{Y}=1})=AUC-0.5$. If the distributions coincide, then $AUC=0.5$ and the value of the statistic will be around $0$. On the other hand, if the distributions are well separated, then $AUC\approx 1$ and the value of the statistic will be around $0.5$. In experiments, we refer to this method as NB AUC.

\section{Theoretical justifications}
\label{Theory}
We first show that Algorithm 1 allows to control the type I error (the probability of rejecting $H_0$ when SCAR is met) in an idealized situation when the set $\widehat{\mathcal{P}}$ coincides with the positive set $\mathcal{P}$. Then we provide some justification for the choice of $\widehat{\cal P}$. 
In order to address the first problem note that probability of rejecting $H_0$ can be written in terms of p-value $\widehat{p}$ as $P(\widehat p < \alpha)$.
The following Theorem indicates that the probability does not exceed $\alpha$, provided that $H_0$ is true.
\setcounter{theorem}{0}
\begin{theorem}
    Assume that SCAR assumption  is met and the algorithm is based  on ${\cal P}$ in place of $\widehat{\cal P}$  . Then distribution of $\widehat p$ is super-uniform i.e.
\[P(\widehat p < t) \leq  t,\quad t\in (0,1).\]
\end{theorem}
\begin{proof}
    Let us denote by $D_{\mathcal{L}}=\{X_i:i\in\mathcal{L}\}$, $D_{\mathcal{P}}=\{X_i:i\in\mathcal{P}\}$ and $D_{\widetilde{\mathcal{L}}}=\{X_i:i\in\widetilde{\mathcal{L}}\}$, where $\widetilde{\mathcal{L}}=\{i:\widetilde{S}_i=1\}$ samples corresponding to  distributions $P_{X|S=1}$, $P_{X|Y=1}$ and $P_{X|\widetilde{S}=1}$, respectively. Test statistics, considered in Algorithm \ref{Alg1} can be written as functions of the samples, i.e.,
    \[
    T_b=T(D_{\widetilde{\mathcal{L}}},D_{\mathcal{P}}), \quad T_0=T(D_{\mathcal{L}},D_{\mathcal{P}}).
    \]
    Under SCAR, $D_{\widetilde{\mathcal{L}}}$ contains conditionally independent observations given 
   $D_{\mathcal{L}}$, generated from $P_{X|\widetilde{S}=1}=P_{X|Y=1}$, whence they are distributionally equal to observations from $P_{X|Y=1}$. The $B+1$ random variables $T_0,T_1,\ldots,T_B$ are exchangeable, i.e. their joint distribution does not change when their positions are randomly ordered. Exchangeability implies that p-value is uniformly distributed on $\{0,1/B,\ldots,B/(B+1)\}$ which implies that $P(\widehat p\leq t)=[t(B+1)]/(B+1) \leq t$, where $[s]$ is integer part of $s$.
\end{proof}



In order to check the soundness of the choice of sample $\widehat {\mathcal P}$ as a substitute of all positive observations,  we consider the idealized scenario in which $s(x)$ is known and $(X_1,Y_1),\ldots, (X_m,Y_m)$ is an iid sequence from $P_{X,Y|S=0}$. Thus, with a slight  abuse of previous notion, $X_1,\ldots, X_m$ correspond to observed unlabeled observations, whereas corresponding $Y_i$ are not observed and $m$ is deterministic sequence corresponding to
expected number of unlabeled observations $m=n(1-c\pi)$. 
We consider $s(X_1),\ldots, s(X_m)$ and denote by $s(X)_{(i)}$ $i$th order statistic in this sequence starting from the largest one i.e. 
\[ s(X)_{(1)}\geq s(X)_{(2)}\ldots\geq s(X)_{(m )} \]
We will consider top $k$ values  $s(X)_{(1)},\ldots s(X)_{(k)}$. We disregard ties assuming in the following that $s(X)$ is continuous random variable.
This corresponds in the algorithm
to considering top $k =n(\pi -\pi c)$ values of $\widehat s(X_i)$ and adding them to labeled observations. 
The above approach is justified by the following Lemma which shows that ordering observations with respect to $s(x)$ is equivalent to ordering with respect to conditional probability
\[\widetilde y(x)=P(Y=1|S=0,X=x).\]
\setcounter{theorem}{0}
\begin{lemma}
\label{Lemma_Ordering}
Assume SCAR assumption is hold. Then for any observations $X_i$ and $X_j$
\[
s(X_i)\geq s(X_j) \iff \widetilde{y}(X_i)\geq \widetilde{y}(X_j) 
\]
\end{lemma}
\begin{proof}
Under SCAR, we have $s(x)=cy(x)$ and thus ordering wrt to $s(x)$ can be replaced by ordering wrt $y(x)$. From Bayes Theorem,
\[ \widetilde y(x)=\frac{P(S=0|Y=1,X=x) y(x)}{ P(S=0|X=x)}=\frac{(1-c)y(x)}{ 1-c y(x)} \]
is increasing function of $y(x)$, which proves the assertion.
\end{proof}

We will establish a bound on the probability that for  the lowest chosen observation  $\widetilde y(X)_{(k)}$  the corresponding value of $Y=1$. To this end define $Y_{[k]}$ as the concomitant value of  $\widetilde y(X)_{(k)}$ i.e. 
\[Y_{[k]}=Y_i\quad \textrm{if}\quad   \widetilde y(X)_{(k)}= \widetilde y(X_i).\]
We define the following function
\[ h(z)= P_{X,Y}(Y=1|\widetilde y(X)=z,S=0).\]
Discussion of the properties of $h(z)$ can be found in the supplement.
The crucial observation is that the following equality holds
\[ P(Y_{[i]}=1| \widetilde y(X)_{(i)}=z)= h(z)\] 
and thus 
\begin{equation}
\label{expected}
P(Y_{[i]}=1)= E h(\widetilde y(X)_{(i)}).
\end{equation}
Let $F$ denote cdf of $h(\widetilde y(X))$ where $X$ is distributed according to $P_{X|S=0}$ i.e. $F(t)=P_{X}(h(\widetilde y(X))\leq t|S=0)$. 
We have the following result.
\begin{theorem}
Assume that $h(z)$ is strictly increasing function. 
\begin{enumerate}[label=(\roman*)]
    \item Let $k=k(m)$ be a sequence such that $k/m\to \alpha$, where $0<\alpha<1$. Moreover, $F$ has continuous density $f$. Then we have  for $m\to\infty$
\begin{equation}
    \label{univariate}
     P(Y_{[k]}=1) =  F^{-1}(1-\alpha) +{\cal O}\Big(\frac{1}{m^{1/2}}\Big) 
\end{equation}
\item For $l\leq k$ we have
\[ P(Y_{[l]}=1)\geq P(Y_{[k]}=1). \]
\item for any $k,l\leq m$ we have
\[ P(Y_{[k]}=1,Y_{[l]}=1)\geq P(Y_{[k]}=1)P(Y_{[l]}=1) \]
\end{enumerate}
\end{theorem}
 \begin{proof}
 Part (i) follows from  Theorem 2.2 (b) in \cite{Bickel1967}) for $k=2$ there ($k$ denotes the order of the moment in the cited paper) and application of Schwarz inequality after noting that $h(\widetilde y(X)_{(k)})$ can be represented as $k$th order statistic from the sequence $F^{-1}(U_1),\ldots,F^{-1}(U_m)$, where $(U_i)$ is iid sample from the uniform distribution. Note that  $h(\widetilde y(X)_{(i)})= (h(\widetilde y(X))_{(i)}$ is  valid  in view of monotonicity of $h$. Proofs of (ii) and (iii) are given in the supplement.

  \end{proof}
 Note that the magnitude of   $F^{-1}(1-\alpha)$ appearing in  (\ref{univariate}) is inherently  related to separability of  $P_{X|Y=0}$ and  $P_{X|Y=1}$.
 In order to  see that recall again that in view of conditional independence of $X$ and $S$ given $Y$ under SCAR we have that  $P_{X|S=0,Y=1}=P_{X|Y=1} $ and $P_{X|S=0,Y=0}=P_{X|Y=0} $.  Thus if $P_{X|Y=0}$ and  $P_{X|Y=1}$ are well separated $h(\widetilde y(X))$ is close to 1  for all positive unlabeled observations, which constitute fraction   $\gamma= (\pi -\pi c)/(1-\pi c)$ of all unlabeled ones. Thus for  $\alpha\leq \gamma$ we have that $F^{-1}(1-\alpha)\approx 1$.
Moreover, it follows from part (i) that, provided the following condition holds 
\begin{equation}
\label{dominance}
F^{-1}(z)\geq z \equiv  z\geq F(z),
\end{equation}
i.e. $h(\widetilde y(X))$ stochastically dominates $[0,1]$-uniformly distributed random variable \cite{BickelDoksum}, that  we have that
\[ P(Y_{[k]} =1)\geq 1-\alpha +{\cal O}(\frac{1}{m}),\]
and analogous result, with $1-\alpha $  replaced by $(1-\alpha)^2$, holds for   probability  $P(Y_{[k]} =1, Y_{[k-1]}=1) $  of two adjacent concomitants. Interestingly, we can  also have more general and simpler result provided (\ref{dominance}) is valid. Note that now the result concerns $k$ concomitants corresponding to to $k$  top order statistics.
\begin{theorem}
 Assume that condition (\ref{dominance}) holds for $F$. Then we have
\begin{equation}
 \label{multi} 
 P(Y_{[1]} =1,Y_{[2]} =1,\ldots, Y_{[k]}=1)\geq \prod_{i=1}^k (1-\frac{i}{m+1})
\end{equation}
\end{theorem}
The proof of this result is given in the supplemental material. Note that, e.g. for $\pi=0.2$, $c=0.8$ and $n=100$, we need to choose additional $k=n(\pi-c\pi) = 100\times 0.04 =4$ observations from $m=100\times (1-0.16)=84$ unlabeled observations. In this case the probability bound  in (\ref{multi})  is $0.887$. However, for large $\pi$ and small $c$, the bound  may become weak. In order to obtain better guarantees, one may choose smaller number of top order statistics than $n(\pi-c\pi)$, focusing on most likely positive observations among unlabeled ones in the modified algorithm.

\section{Experiments}
We analyze the effectiveness of the proposed testing procedure and compare the performance of 4 statistics: KL, KLCOV, KS and NB AUC. 
As evaluation measures, we consider: type I error (probability of rejecting $H_0$ when $H_0$ is true), which should not exceed assumed significance level $\alpha$ and power of the test (probability of rejecting $H_0$ when $H_1$ is true). 
In particular, we aim to answer the following research questions. (1) Do the tests control for a type I error? (2) Which of the proposed statistics has the greatest power ? (3) How does the power depend on various factors such as: the sample size, dependence between features or the discrepancy between data distribution and $H_0$ distribution?
In the experiments we set $B=300$, $\alpha=0.05$. Moreover, Random Forest classifier \cite{Breiman2001} was used as a base learner to estimate $s(x)$. To estimate the probability of rejecting $H_0$ we repeated experiments $500$ times.

\begin{figure*}
\centering
    \begin{tabular}{c c c c}
      \includegraphics[width=0.2\textwidth]{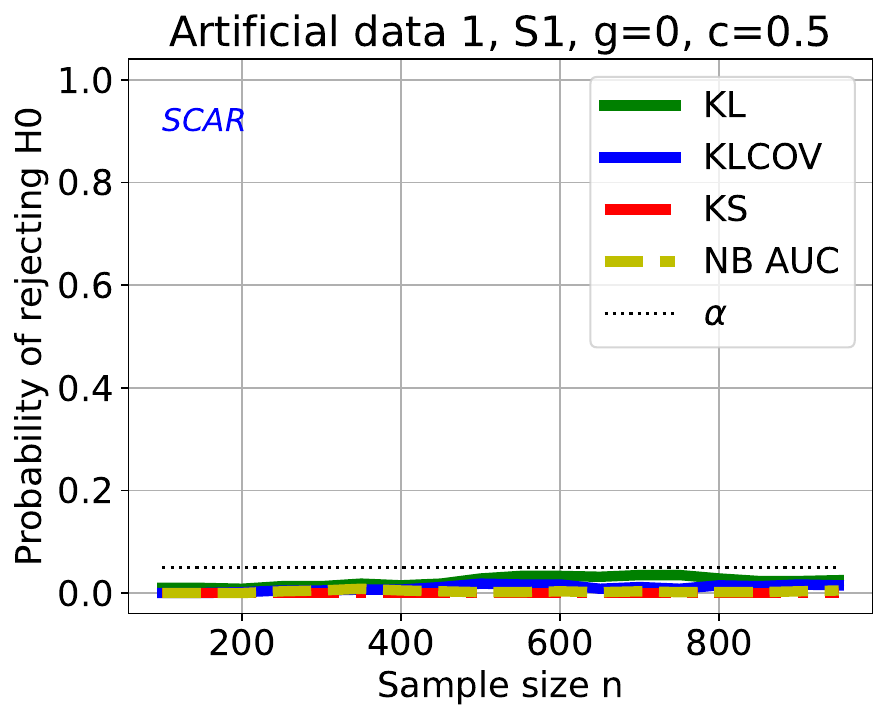}
      \includegraphics[width=0.2\textwidth]{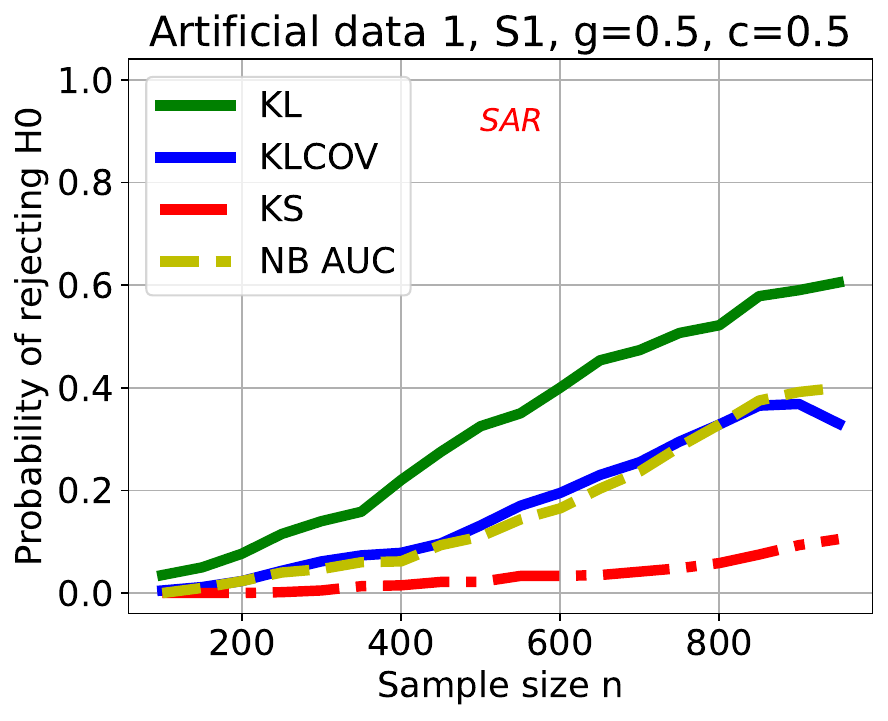}
      \includegraphics[width=0.2\textwidth]{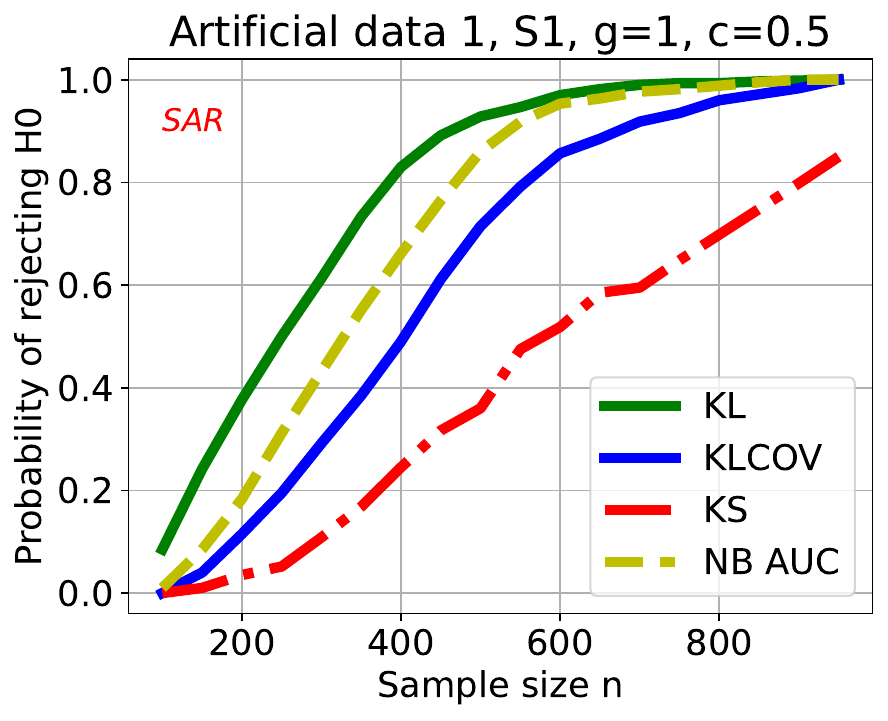}
      \includegraphics[width=0.2\textwidth]{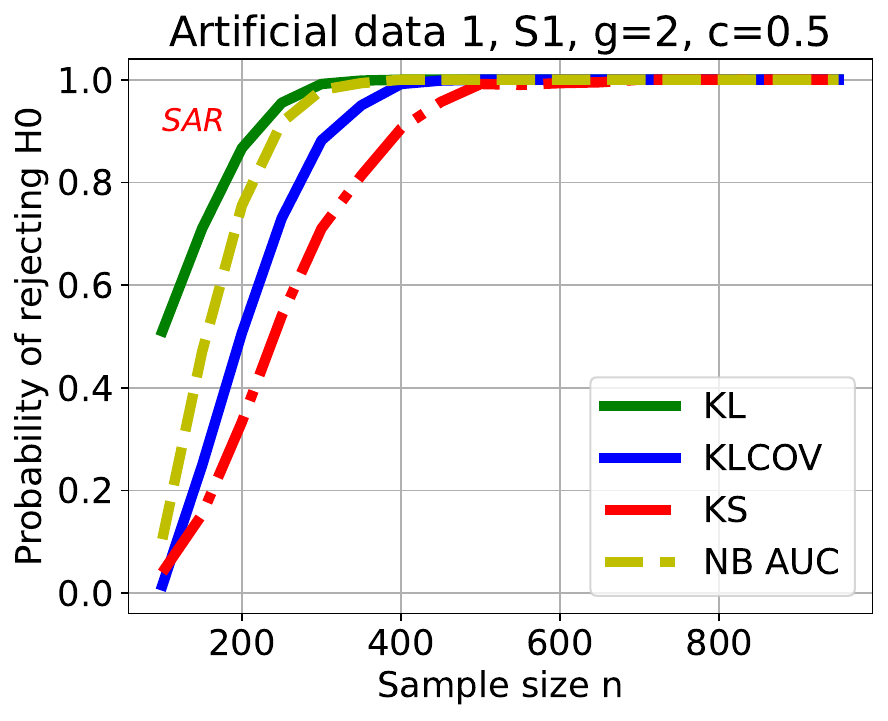}\\
      \includegraphics[width=0.2\textwidth]{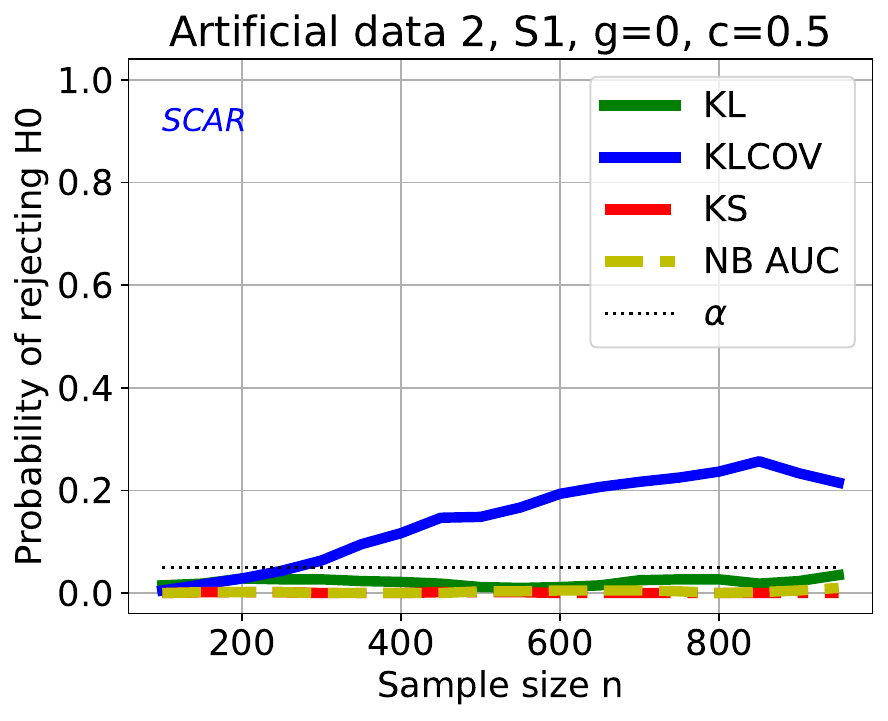}
      \includegraphics[width=0.2\textwidth]{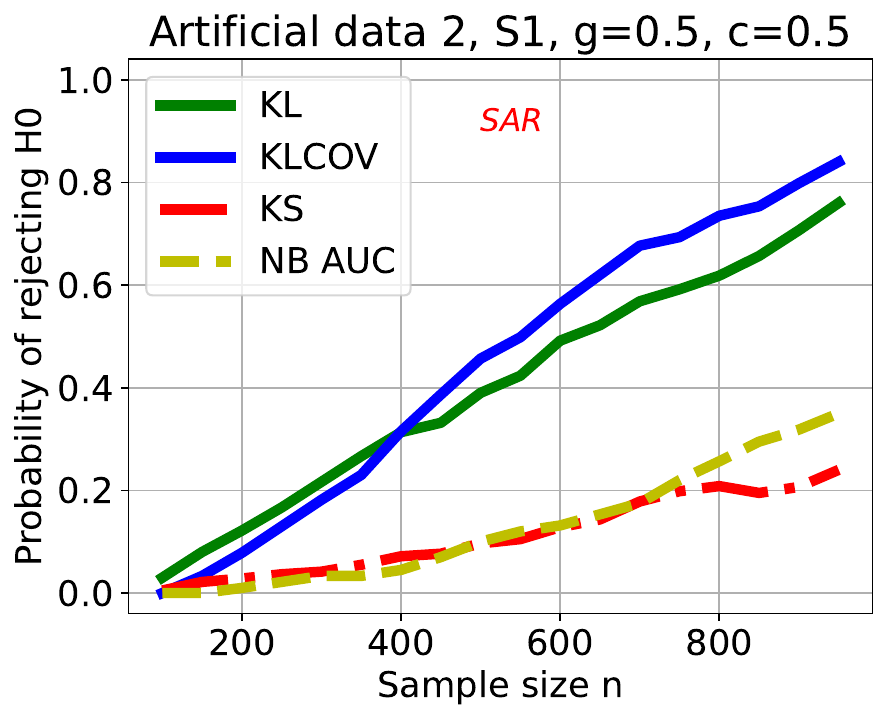}
      \includegraphics[width=0.2\textwidth]{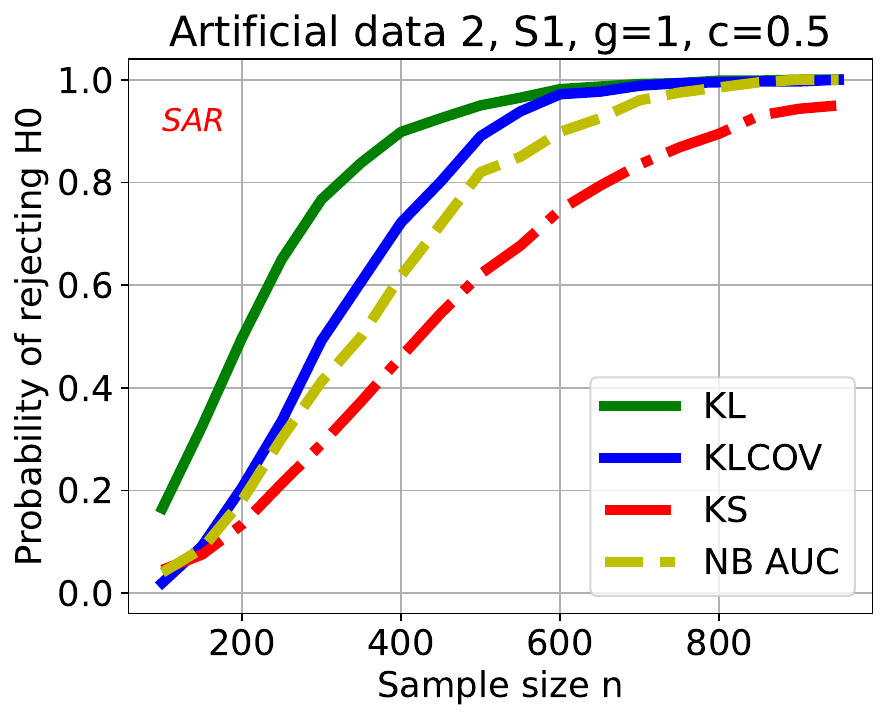}
      \includegraphics[width=0.2\textwidth]{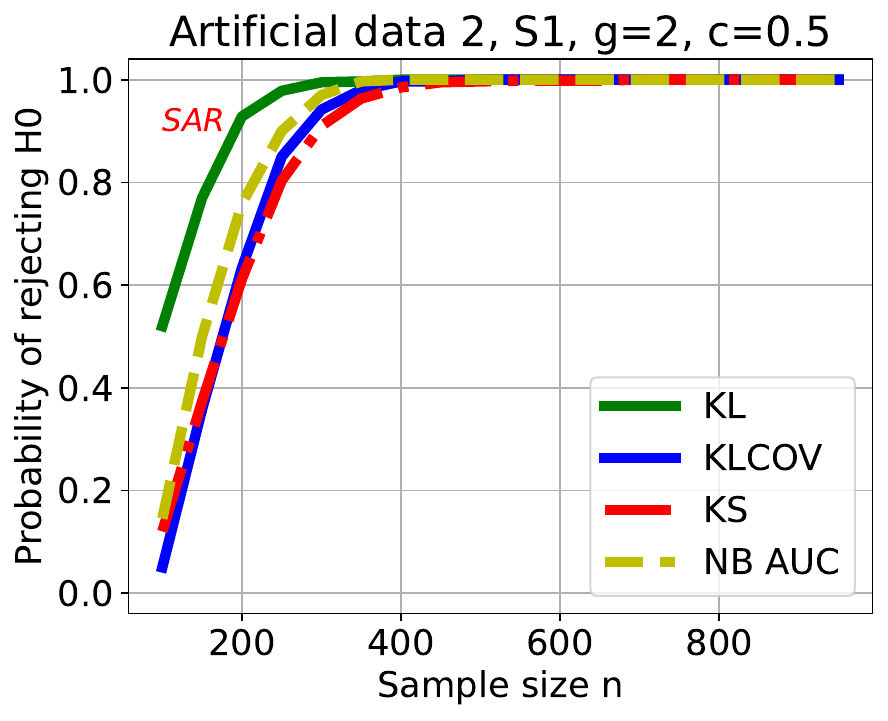}\\
    \end{tabular}
    \caption{Probability of rejecting $H_0$ with respect to sample size $n$ for artificial data sets 1 and 2, labeling strategy S1 and for $c=0.5$. Value $g=0$ corresponds to $H_0$ and $g>0$ to $H_1$.}
    \label{fig:art_n_c05}
\end{figure*}

\subsection{Datasets}
In experiments we used 4 popular tabular datasets (Breast Cancer, Wdbc, Banknote and Segment) \cite{Dua2019} and 3 image datasets (CIFAR 10, USPS and Fashion) \cite{PyTorch19}. Details about preprocessing the datasets are described in the supplement; Table 1 in the supplement contains summary statistics. 
Moreover, we used two artificial datasets (Art1 and Art2), which are obtained as follows. In Art1, we first generate $Y$ from the Bernoulli distribution with a success probability of $0.5$. Then we generate feature vectors $X$ from the distributions $P_{X|Y=0}\sim N(0,I)$ and $P_{X|Y=1}\sim N(b,I)$, where $b=(1 ,\ldots,1)$.
In Art2, the feature vectors are generated from the distributions $P_{X|Y=0}\sim N(0,I)$ and $P_{X|Y=1}\sim N(b,\Sigma)$, where $\Sigma[i,j]=0.5^{|i-j|}$. In Art1 we assume independence of features, while in Art2, the features are dependent and additionally covariance matrices in the positive and negative classes are different.


\subsection{Labeling strategies}
Given a dataset with binary class variable $Y$, we artificially generate PU data using various labeling strategies. All negative observations are assigned to unlabeled subset. From the positive observations we randomly select those that will be labeled with probability $e(x)=P(S=1|X=x,Y=1)$, whereas the remaining observations are assigned to unlabeled set. The following strategies are considered.
\begin{itemize}
\item [{\bf S0.}] Propensity score is constant $e(X_i)=c$. 
\item [{\bf S1.}] Propensity score  $e(X_i)=\sigma(g\cdot X_{i,1})$, where $X_{i,1}$ is a value of the first feature, for $i$-th observation and $\sigma(s)=\exp(s)/(1+\exp(s))$. 
\item [{\bf S2.}] Propensity score  $e(X_i)=\sigma(g\cdot X_i^{T}\beta^*+a)$.
\item [{\bf S3.}] Propensity score $e(X_i)=[\sigma(g\cdot X_i^{T}\beta^*+a)]^{10}$.
\end{itemize}
Strategy S0 is used to analyze type I error for the methods. 
Strategies S2 and S3 were already used in papers on instance-based PU learning \cite{Gong2021, FurmanczykECAI2023}.
Parameter vector $\beta^*$ is obtained from logistic regression model fitted on the fully labeled data, i.e., assuming the knowledge of $Y$.
Parameter $g\geq 0$  controls how far we are from the null hypothesis $H_0$. Note that the value $g=0$, corresponds to SCAR, i.e. the propensity function is constant. The value of $g>0$ corresponds to the SAR situation and by increasing $g$, we move away from $H_0$. Moreover, parameter $g$ controls how much $X_i$ affects the propensity score. 
Parameter $a$ is determined to control the value of labeling frequency $c=P(S=1|Y=1)$. Value of $a$ is calculated for the previously found parameter $\beta^*$ and fixed $g$. We report the results for S1-S2 and $c=0.5$, the results for $c=0.3,0.7$ as well as for S3 are given in the supplement.

\subsection{Discussion}
\subsubsection{Controlling type I error}
Figure \ref{fig:art_n_c05} (top left and bottom left panels) indicates that, in the case of artificial data, all methods control for type I error when the features are independent (Art1). In the case of dependencies (Art2), KL, KS and NB AUC work correctly, i.e., they do not exceed the assumed significance level $\alpha=0.05$. For KLCOV, the probability of rejecting $H_0$  exhibits undesirable increase with the sample size.
KS and NB AUC control for the type I error for all 7 real data sets (Table \ref{tab:Real_S0}). KL and KLCOV exceed $\alpha$ for 2 and 3 datasets, respectively. In particular, KLCOV always rejects $H_0$ for Banknote data, which  is due to the lack of robustness to the inaccurate estimation of the positive set $\mathcal{P}$. 
Indeed, when assuming knowledge of set $\mathcal{P}$, the type I error does not exceed $\alpha$ for this method.
In summary, the KL and NB AUC perform conservatively for all data sets and they should be recommended if controlling  the type I error is our important objective.
Conclusions remain similar for $c=0.3, 0.7$ (Tables 2, 3 in the supplement).
\subsubsection{Power of the tests}
As expected, for artificial datasets, the power of the tests increases when the number of observations increases (Figure \ref{fig:art_n_c05}).  KL method has the largest power, followed by KLCOV. Importantly, however, KLCOV does not control for type I error for Art2, so analyzing the power may be misleading for this method. KS and NB AUC converge more slowly to $1$, but this is the price for effectively controlling the type I error. The power also increases when the $g$ parameter is increased, which is natural because a larger $g$ indicates more significant deviation from $H_0$ (Figure \ref{fig:art_n_c05}). For both artificial datasets and $g=2$, the power for all methods approaches $1$ for relatively small sample size $500$, whereas for $g=0.5, 1$, we need significantly more observations to achieve this level of power.

For real datasets, we also see that the power increases as the $g$ parameter increases (Figure \ref{fig:real_g_c05} and Table \ref{tab:Real_S1_c05}).
Figure \ref{fig:real_g_c05} shows that the KS method usually achieves the highest power. The exception is the Wdbc dataset and S2, for which KL and KLCOV have the highest power (Figure \ref{fig:real_g_c05}), but for these methods the type I error is significantly exceeded, so they should not be taken into account in the comparison for this particular dataset.
Among methods that properly control type 1 error (KS and NB AUC), KS achieves greater power for most datasets, labeling schemes, and $g$ parameter values (Table \ref{tab:Real_S1_c05}). For example, for S1, KS is the winner 12 out of 14 times, for S2 $9/14$ times. Conclusions remain similar for $c=0.3, 0.7$ an for S3 (Tables 4-6 in the supplement). 

We also examined the robustness of the testing procedure to the class prior estimation error. The results in Table 7 (supplement) indicate that overestimation of $\pi$ has a greater negative impact than underestimation.
For overestimated $\pi$ we observe that very often type I error exceeds the $\alpha$ level,
which is due to the fact that in this case, set $\mathcal{P}$ contains, in addition to true positive observations, too many negative observations.



\begin{table}
\centering
{\footnotesize
\caption{Type I error (probability of rejecting $H_0$ when $H_0$ is true, also called observed level of significance) for $c=0.5$. Cases in which the type I error exceeds the assumed level $\alpha=0.05$ are marked in red.}
\label{tab:Real_S0}
\begin{tabular}{l||llll}
\toprule 
Dataset &	KL & KLCOV & KS & NB AUC   \\
\midrule
Breast-w &	0.01 $\pm$ 0.01 &	0.0 $\pm$ 0.0 &	0.0 $\pm$ 0.0 &	0.0 $\pm$ 0.0 \\
Wdbc &	\rd{0.18 $\pm$ 0.04} &	\rd{0.2 $\pm$ 0.04} &	0.0 $\pm$ 0.0 &	0.02 $\pm$ 0.01 \\
Banknote &	0.03 $\pm$ 0.02 &	\rd{1.0 $\pm$ 0.0} &	0.04 $\pm$ 0.02 &	0.0 $\pm$ 0.0 \\
Segment &	\rd{0.12 $\pm$ 0.03} &	\rd{0.15 $\pm$ 0.04} &	0.04 $\pm$ 0.02 &	0.01 $\pm$ 0.01 \\
CIFAR10 &	0.01 $\pm$ 0.01 &	0.03 $\pm$ 0.02 &	0.05 $\pm$ 0.02 &	0.0 $\pm$ 0.0 \\
USPS &	0.02 $\pm$ 0.01 &	0.05 $\pm$ 0.03 &	0.02 $\pm$ 0.01 &	0.0 $\pm$ 0.0 \\
Fashion &	0.01 $\pm$ 0.01 &	0.03 $\pm$ 0.02 &	0.02 $\pm$ 0.01 &	0.0 $\pm$ 0.0 \\
\bottomrule
\end{tabular}
}
\end{table}

\begin{table}
\centering
\caption{ Power of the tests (probability of rejecting $H_0$ when $H_1$ is true) fro $c=0.5$ and labeling schemes S1 and S2. The results for the winning method among methods KS and NB AUC are in bold (KL and KLCOV are excluded from the comparison because they do not control for type I error for some datasets).}
\label{tab:Real_S1_c05}
\begin{tabular}{l|l||llll}
\toprule 
 \multicolumn{6}{c}{labeling scheme S1}  \\
\toprule
Dataset &	$g$ &	KL & KLCOV & KS & NB AUC   \\
\midrule
Breast &	1 &	0.15 $\pm$ 0.04 &	0.02 $\pm$ 0.01 &	\textbf{0.24 $\pm$ 0.04} &	0.02 $\pm$ 0.01 \\
 &	2 &	0.79 $\pm$ 0.04 &	0.17 $\pm$ 0.04 &	\textbf{0.74 $\pm$ 0.04} &	0.54 $\pm$ 0.05 \\
\midrule
Wdbc &	1 &	0.79 $\pm$ 0.04 &	0.3 $\pm$ 0.05 &	\textbf{0.91 $\pm$ 0.03} &	0.8 $\pm$ 0.04 \\
 &	2 &	1.0 $\pm$ 0.0 &	0.66 $\pm$ 0.05 &	\textbf{1.0 $\pm$ 0.0} &	\textbf{1.0 $\pm$ 0.0} \\
\midrule
Banknote &	1 &	1.0 $\pm$ 0.0 &	1.0 $\pm$ 0.0 &	\textbf{1.0 $\pm$ 0.0} &	\textbf{1.0 $\pm$ 0.0} \\
 &	2 &	1.0 $\pm$ 0.0 &	1.0 $\pm$ 0.0 &	\textbf{1.0 $\pm$ 0.0} &	\textbf{1.0 $\pm$ 0.0} \\
\midrule
Segment &	1 &	0.15 $\pm$ 0.04 &	0.19 $\pm$ 0.04 &	\textbf{0.4 $\pm$ 0.05} &	0.27 $\pm$ 0.04 \\
 &	2 &	0.18 $\pm$ 0.04 &	0.26 $\pm$ 0.04 &	\textbf{0.95 $\pm$ 0.02} &	0.83 $\pm$ 0.04 \\
\midrule
CIFAR10 &	1 &	0.07 $\pm$ 0.03 &	0.0 $\pm$ 0.0 &	\textbf{0.41 $\pm$ 0.05} &	0.0 $\pm$ 0.0 \\
 &	2 &	0.19 $\pm$ 0.04 &	0.01 $\pm$ 0.01 &	\textbf{0.57 $\pm$ 0.05} &	0.0 $\pm$ 0.0 \\
\midrule
USPS &	1 &	0.96 $\pm$ 0.02 &	0.92 $\pm$ 0.03 &	0.87 $\pm$ 0.03 &	\textbf{0.92 $\pm$ 0.03} \\
 &	2 &	0.96 $\pm$ 0.02 &	0.9 $\pm$ 0.03 &	\textbf{0.94 $\pm$ 0.02} &	0.91 $\pm$ 0.03 \\
\midrule
Fashion &	1 &	0.99 $\pm$ 0.01 &	0.95 $\pm$ 0.02 &	0.99 $\pm$ 0.01 &	\textbf{1.0 $\pm$ 0.0} \\
 &	2 &	1.0 $\pm$ 0.0 &	1.0 $\pm$ 0.0 &	\textbf{1.0 $\pm$ 0.0} &	\textbf{1.0 $\pm$ 0.0} \\
\bottomrule
\toprule 
 \multicolumn{6}{c}{labeling scheme S2}  \\
\toprule
Dataset &	$g$ &	KL & KLCOV & KS & NB AUC   \\
\midrule

Breast &	1 &	0.95 $\pm$ 0.02 &	0.37 $\pm$ 0.05 &	\textbf{0.95 $\pm$ 0.02} &	0.89 $\pm$ 0.03 \\
 &	2 &	1.0 $\pm$ 0.0 &	1.0 $\pm$ 0.0 &	\textbf{1.0 $\pm$ 0.0} &	\textbf{1.0 $\pm$ 0.0} \\
 \midrule
Wdbc &	1 &	0.59 $\pm$ 0.05 &	0.52 $\pm$ 0.05 &	\textbf{0.2 $\pm$ 0.04} &	0.15 $\pm$ 0.04 \\
 &	2 &	0.89 $\pm$ 0.03 &	0.66 $\pm$ 0.05 &	\textbf{0.79 $\pm$ 0.04} &	0.65 $\pm$ 0.05 \\
\midrule
Banknote &	1 &	0.8 $\pm$ 0.04 &	1.0 $\pm$ 0.0 &	0.46 $\pm$ 0.05 &	\textbf{1.0 $\pm$ 0.0} \\
 &	2 &	1.0 $\pm$ 0.0 &	1.0 $\pm$ 0.0 &	0.99 $\pm$ 0.01 &	\textbf{1.0 $\pm$ 0.0} \\
\midrule
Segment &	1 &	0.06 $\pm$ 0.02 &	0.13 $\pm$ 0.03 &	0.06 $\pm$ 0.02 &	\textbf{0.21 $\pm$ 0.04} \\
 &	2 &	0.12 $\pm$ 0.03 &	0.16 $\pm$ 0.04 &	0.3 $\pm$ 0.05 &	\textbf{0.79 $\pm$ 0.04} \\
\midrule
CIFAR10 &	1 &	0.29 $\pm$ 0.05 &	0.2 $\pm$ 0.04 &	0.05 $\pm$ 0.02 &	\textbf{0.08 $\pm$ 0.03} \\
 &	2 &	0.9 $\pm$ 0.03 &	0.88 $\pm$ 0.03 &	\textbf{0.83 $\pm$ 0.04} &	0.79 $\pm$ 0.04 \\
\midrule
USPS &	1 &	0.48 $\pm$ 0.05 &	0.48 $\pm$ 0.05 &	\textbf{0.44 $\pm$ 0.05} &	0.37 $\pm$ 0.05 \\
 &	2 &	1.0 $\pm$ 0.0 &	1.0 $\pm$ 0.0 &	\textbf{0.99 $\pm$ 0.01} &	0.98 $\pm$ 0.01 \\
\midrule
Fashion &	1 &	0.53 $\pm$ 0.05 &	0.41 $\pm$ 0.05 &	\textbf{0.65 $\pm$ 0.05} &	0.26 $\pm$ 0.04 \\
 &	2 &	0.99 $\pm$ 0.01 &	0.98 $\pm$ 0.01 &	\textbf{1.0 $\pm$ 0.0} &	\textbf{1.0 $\pm$ 0.0} \\
\bottomrule
\end{tabular}
\end{table}

\begin{figure}[ht!]
\centering
    \begin{tabular}{cc}
      \includegraphics[width=0.35\textwidth]{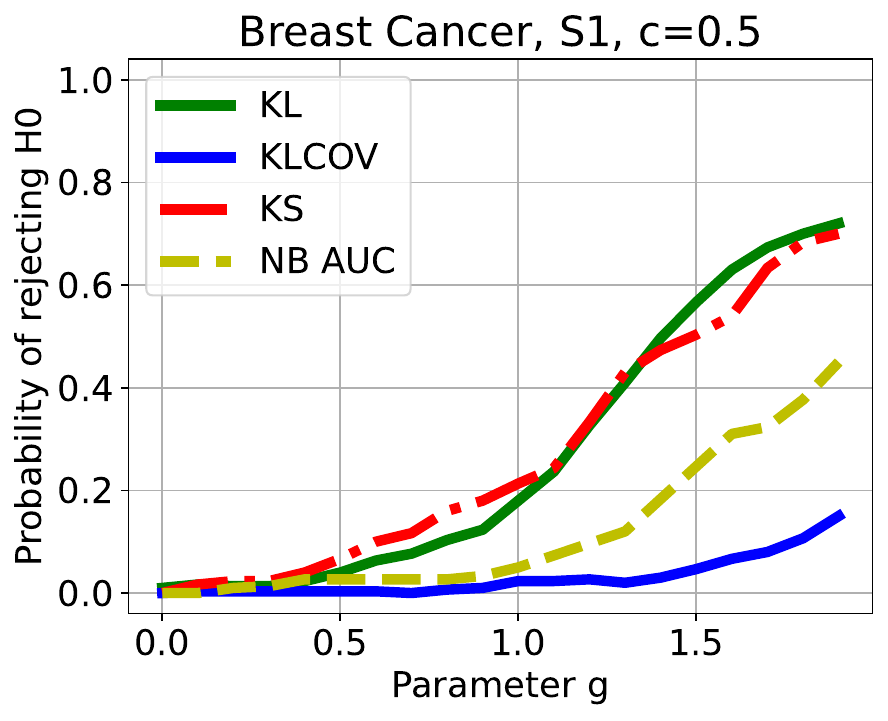} &
      \includegraphics[width=0.35\textwidth]{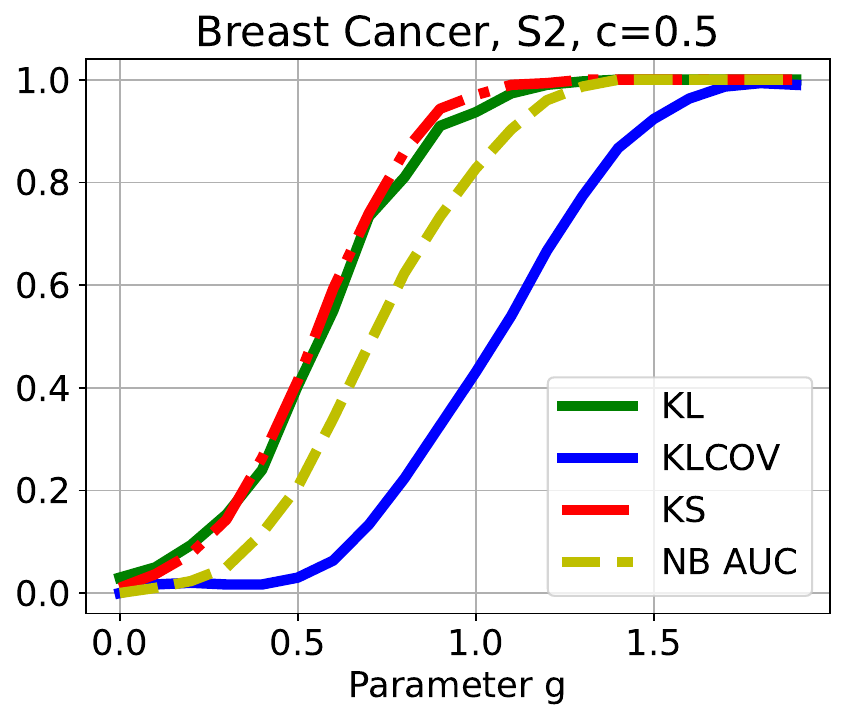}  \\
     \includegraphics[width=0.35\textwidth]{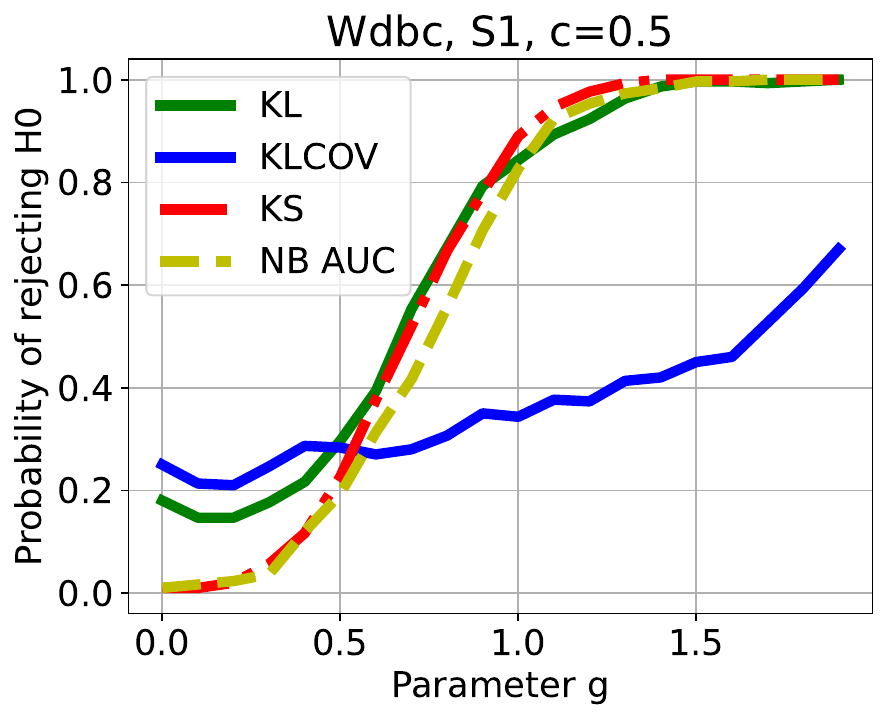} &
      \includegraphics[width=0.35\textwidth]{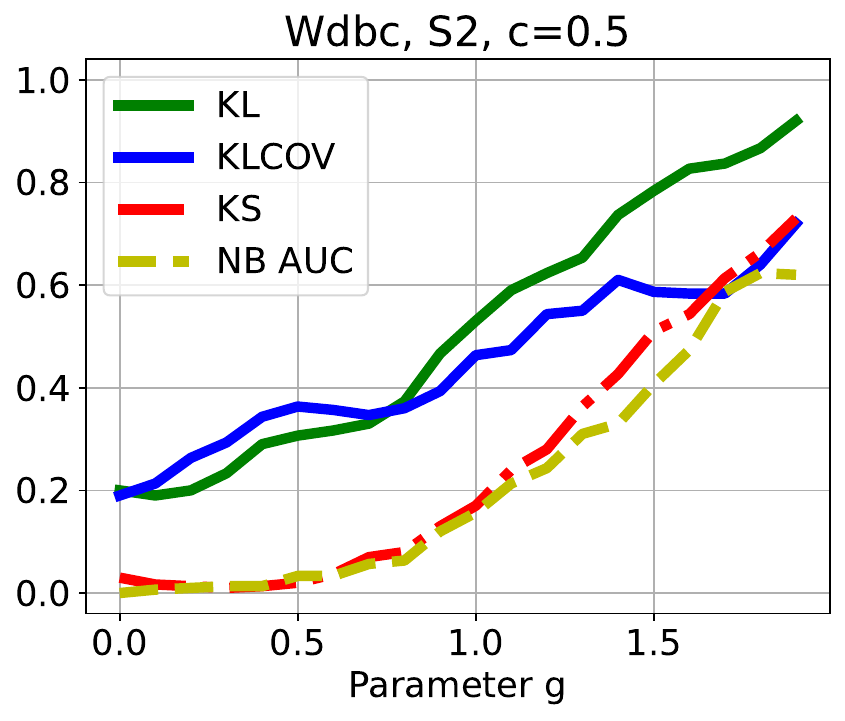}  \\
      \includegraphics[width=0.35\textwidth]{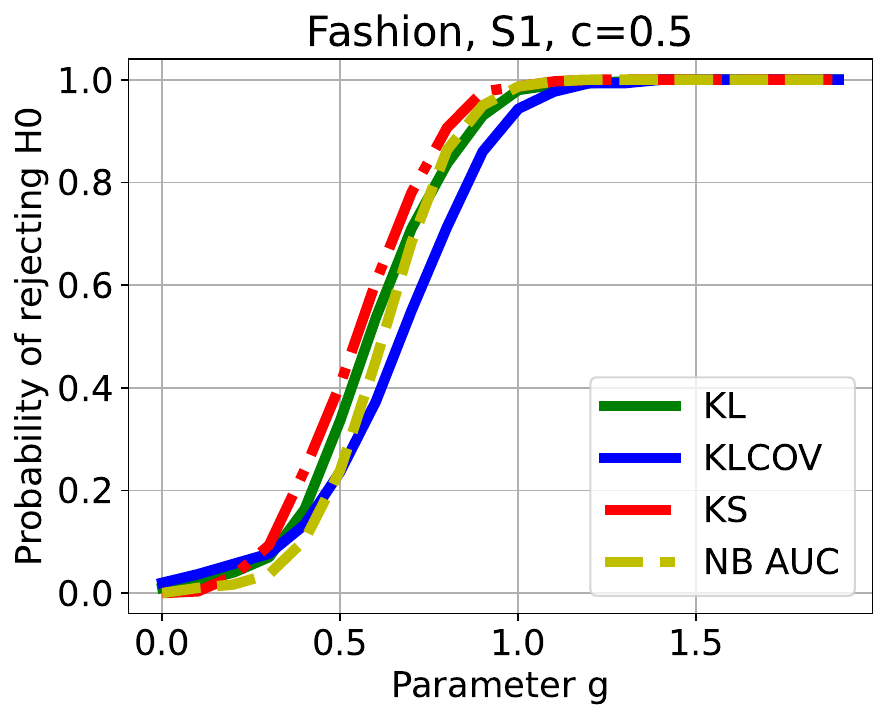} &
      \includegraphics[width=0.35\textwidth]{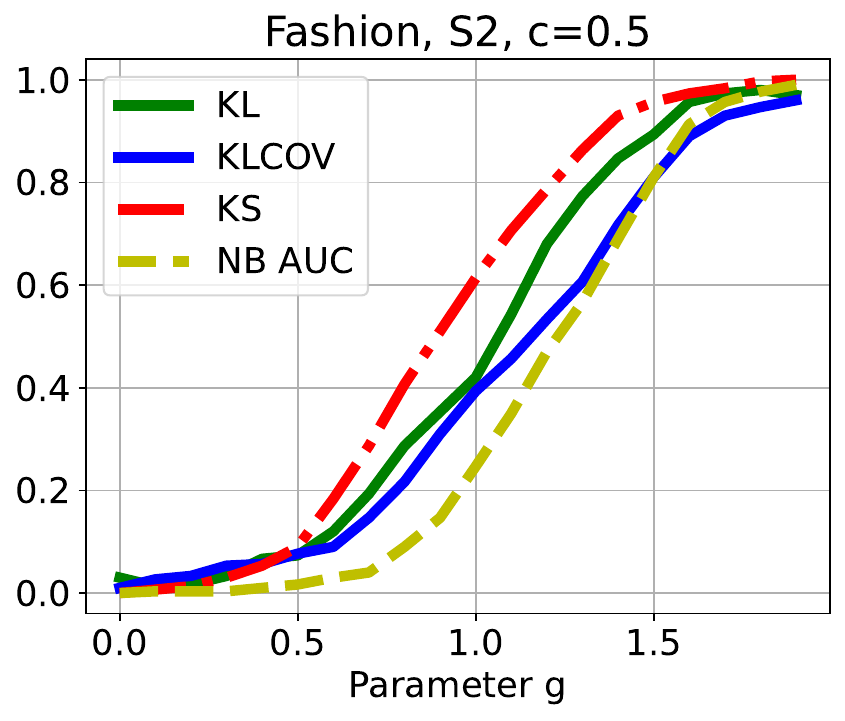}  \\
            \includegraphics[width=0.35\textwidth]{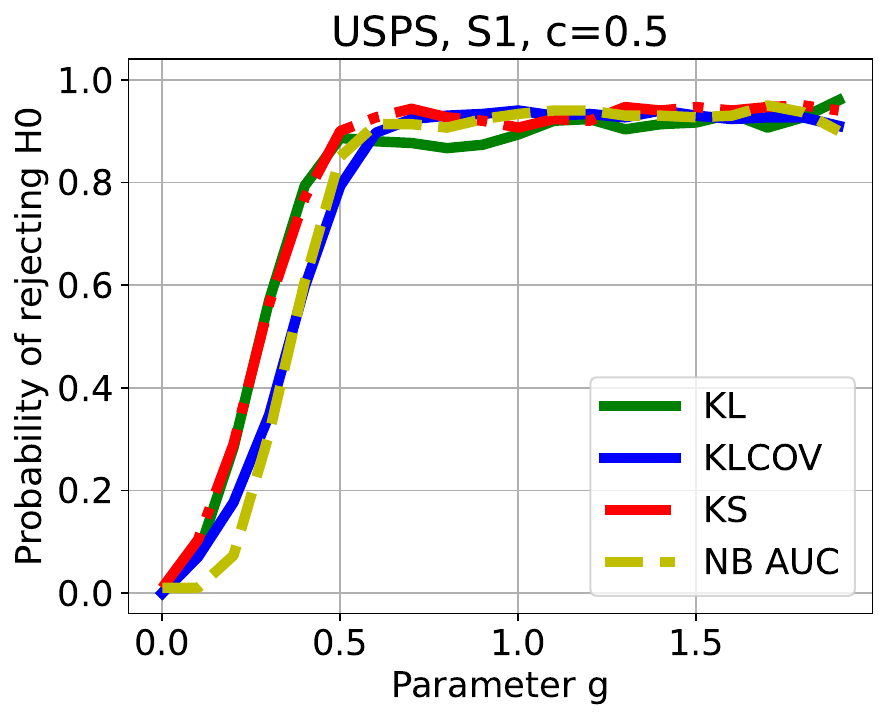} &
      \includegraphics[width=0.35\textwidth]{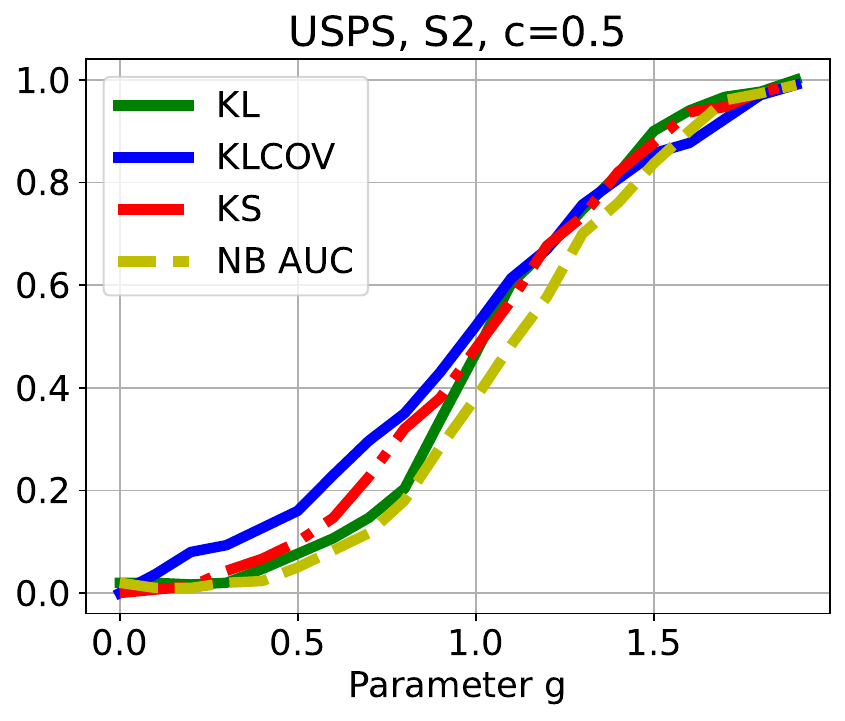}  \\

      \end{tabular}
    \caption{Probability of rejecting $H_0$ with respect to parameter $g$ for selected tabular and image datasets, for $c=0.5$. Value $g=0$ corresponds to $H_0$ (SCAR), whereas $g>0$ to $H_1$ (SAR).}
    \label{fig:real_g_c05}
\end{figure}

\section{Conclusions}
Using the proposed method, it is possible to decide whether PU data correspond the SCAR or SAR assumption, contolling type I error. The method is of significant practical importance, because it allows to choose between  using more computationally expensive SAR algorithms or simpler methods based on SCAR. In many real applications, the impact of features on propensity score may be negligible and then SCAR algorithms are clearly a better choice. Theoretical results justify the method of estimating the set of positive observations and show that if it is estimated correctly, controlling the type I error is actually possible. In future work, other positive set estimation methods can be considered that do not assume knowledge of class prior and are based, for example, on FDR control. The proposed procedure is generic and can be combined with any classifier and test statistic that meets the general conditions. Among the investigated test statistics, we recommend using the Kolmogorov-Smirnov statistic, which properly controls for type I error while still having high power.

\bibliographystyle{unsrt}
\bibliography{References}

\begin{thebibliography}{10}

\bibitem{BekkerDavis2020}
J.~Bekker and J.~Davis.
\newblock Learning from positive and unlabeled data: a survey.
\newblock {\em Machine Learning}, 109:719--760, 2020.

\bibitem{ElkanNoto2008}
C.~Elkan and K.~Noto.
\newblock Learning classifiers from only positive and unlabeled data.
\newblock In {\em Proceedings of the 14th ACM SIGKDD International Conference
  on Knowledge Discovery and Data Mining}, KDD '08, pages 213--220, 2008.

\bibitem{Gong2022}
C.~Gong, M.~I. Zulfiqar, C.~Zhang, S.~Mahmood, and J.~Yang.
\newblock A recent survey on instance-dependent positive and unlabeled
  learning.
\newblock {\em Fundamental Research}, 2022.

\bibitem{Chiaroni2018}
F.~Chiaroni, M-C. Rahal, N.~Hueber, and F.~Dufaux.
\newblock Learning with a generative adversarial network from a positive
  unlabeled dataset for image classification.
\newblock In {\em Proceedings of the 25th IEEE International Conference on
  Image Processing}, ICIP'18, pages 1--6, 2018.

\bibitem{Luo2018}
Y.~Luo, S.~Cheng, C.~Liu, , and F.~Jiang.
\newblock Pu-learning in payload-based web anomaly detection.
\newblock In {\em Proceedings of the Third Conference on Security of Smart
  Cities, industrial Control Systems and Communications}, SSIC'2018, pages
  1--5, 2018.

\bibitem{Perini2023}
L.~Perini, V.~Vercruyssen, and J.~Davis.
\newblock Learning from positive and unlabeled multi-instance bags in anomaly
  detection.
\newblock In {\em Proceedings of the 29th ACM SIGKDD Conference on Knowledge
  Discovery and Data Mining}, KDD '23, page 1897–1906, 2023.

\bibitem{Sechidis2017}
K.~Sechidis, M.~Sperrin, E.~S. Petherick, M.~Luján, and G.~Brown.
\newblock Dealing with under-reported variables: An information theoretic
  solution.
\newblock {\em International Journal of Approximate Reasoning}, 85:159 -- 177,
  2017.

\bibitem{Li2021}
F.~Li, S.~Dong, A.~Leier, M.~Han, X.~Guo, J.~Xu, X.~Wang, S.~Pan, C.~Jia,
  Y.~Zhang, G.~Webb, L.~J.~M. Coin, C.~Li, and J.~Song.
\newblock Positive-unlabeled learning in bioinformatics and computational
  biology: a brief review.
\newblock {\em Briefings in Bioinformatics}, 23(1), 2021.

\bibitem{ChenLiuWangZhaoWu2020}
H.~Chen, F.~Liu, Y.~Wang, L.~Zhao, and H.~Wu.
\newblock A variational approach for learning from positive and unlabeled data.
\newblock In {\em Proceedings of the International Conference on Neural
  Information Processing Systems}, NIPS'20, pages 14844--14854, 2020.

\bibitem{Zhao2022}
Y.~Zhao, Q.~Xu, Y.~Jiang, P.~Wen, and Q.~Huang.
\newblock Dist-pu: Positive-unlabeled learning from a label distribution
  perspective.
\newblock In {\em Proceedings of the Conference on Computer Vision and Pattern
  Recognition}, CVPR'22, pages 14461--14470, 2022.

\bibitem{li2022who}
Changchun Li, Ximing Li, Lei Feng, and Jihong Ouyang.
\newblock Who is your right mixup partner in positive and unlabeled learning.
\newblock In {\em Proceedings of the 10th International Conference on Learning
  Representations}, 2022.

\bibitem{Ramaswamy2016}
H.~Ramaswamy, C.~Scott, and A.~Tewari.
\newblock Mixture proportion estimation via kernel embeddings of distributions.
\newblock In {\em Proceedings of The 33rd International Conference on Machine
  Learning}, volume~48, pages 2052--2060, 2016.

\bibitem{Jainetal2016}
S.~Jain, M.~White, and P.~Radivojac.
\newblock Estimating the class prior and posterior from noisy positives and
  unlabeled data.
\newblock In {\em Proceedings of the 30th International Conference on Neural
  Information Processing Systems}, page 2693–2701, 2016.

\bibitem{BekkerAAAI18}
J.~Bekker and J.~Davis.
\newblock Estimating the class prior in positive and unlabeled data through
  decision tree induction.
\newblock In {\em Proceedings of the 32th {AAAI} Conference on Artificial
  Intelligence}, pages 1--8, 2018.

\bibitem{Lazeckaetal2021}
M.~{\L}azecka, J.~Mielniczuk, and P.~Teisseyre.
\newblock Estimating the class prior for positive and unlabelled data via
  logistic regression.
\newblock {\em Advances in Data Analysis and Classification}, 15:1039--1068,
  2021.

\bibitem{ICCS2020}
P.~Teisseyre, J.~Mielniczuk, and M.~{\L}azecka.
\newblock Different strategies of fitting logistic regression for positive and
  unlabelled data.
\newblock In {\em Proceedings of Intrernational Conference on Computational
  Science}, ICCS'20, pages 1--14, 2020.

\bibitem{BekkerRobberechtsDavis2019}
J.~Bekker, P.~Robberechts, and J.~Davis.
\newblock {B}eyond the {S}elected {C}ompletely {A}t {R}andom {A}ssumption for
  {L}earning from {P}ositive and {U}nlabeled {D}ata.
\newblock In {\em {P}roceedings of the 2019 {E}uropean {C}onference on
  {M}achine {L}earning and {P}rinciples and {P}ractice of {K}nowledge
  {D}iscovery in {D}atabases}, ECML'19, pages 71--85, 2019.

\bibitem{Gerych2022}
W.~Gerych, T.~Hartvigsen, L.~Buquicchio, E.~Agu, and E.~Rundensteiner.
\newblock Recovering the propensity score from biased positive unlabeled data.
\newblock In {\em Proceedings of the AAAI Conference on Artificial
  Intelligence}, AAAI'22, pages 6694--6702, 2022.

\bibitem{Gong2021}
C.~Gong, Q.~Wang, T.~Liu, B.~Han, J.~You, J.~Yang, and D.~Tao.
\newblock Instance-dependent positive and unlabeled learning with labeling bias
  estimation.
\newblock {\em IEEE Trans Pattern Anal Mach Intell}, pages 1--16, 2021.

\bibitem{FurmanczykECAI2023}
K.~Furma\'nczyk, J.~Mielniczuk, W.~Rejchel, and P.~Teisseyre.
\newblock Double logistic regression approach to biased positive-unlabeled
  data.
\newblock In {\em Proceedings of the European Conference on Artificial
  Intelligence}, ECAI'23, pages 764--771, 2023.

\bibitem{Verreet2023}
V.~Verreet, L.~De~Raedt, and J.~Bekker.
\newblock Modeling {PU} learning using probabilistic logic programming.
\newblock {\em Machine Learning}, pages 1--22, 2023.

\bibitem{Kato2019}
M.~Kato, T.~Teshima, and J.~Honda.
\newblock Learning from positive and unlabeled data with a selection bias.
\newblock In {\em Proceedings of the 7th International Conference on Learning
  Representations}, pages 1--12, 2019.

\bibitem{NaVAE}
B.~Na, H.~Kim, K.~Song, W.~Joo, Y-Y. Kim, and I.~Moon.
\newblock Deep generative positive-unlabeled learning under selection bias.
\newblock In {\em Proceedings of CIKM'20}, CIKM '20, pages 1155--–1164, New
  York, NY, USA, 2020. ACM.

\bibitem{ward2009}
G.~Ward, T.~Hastie, S.~Barry, J.~Elith, and J.~Leathwick.
\newblock Presence-only data and the {EM} algorithm.
\newblock {\em Biometrics}, 65:554--563, 2009.

\bibitem{uPU}
Marthinus~C du~Plessis, Gang Niu, and Masashi Sugiyama.
\newblock Analysis of learning from positive and unlabeled data.
\newblock In Z.~Ghahramani, M.~Welling, C.~Cortes, N.~Lawrence, and K.Q.
  Weinberger, editors, {\em Advances in Neural Information Processing Systems},
  volume~27, pages 703--711. Curran Associates, Inc., 2014.

\bibitem{PUSB}
Ryuichi Kiryo, Gang Niu, Marthinus~C. du~Plessis, and Masashi Sugiyama.
\newblock Positive-unlabeled learning with non-negative risk estimator.
\newblock In {\em Proceedings of the NIPS'17}, NIPS'17, pages 1674--–1684,
  Red Hook, NY, USA, 2017. Curran Associates Inc.

\bibitem{Cover2006}
T.~M. Cover and J.~A. Thomas.
\newblock {\em Elements of Information Theory (Wiley Series in
  Telecommunications and Signal Processing)}.
\newblock Wiley-Interscience, 2006.

\bibitem{Belghazietal2018}
M.~I. Belghazi, A.~Baratin, S.~Rajeshwar, S.~Ozair, Y.~Bengio, A.~Courville,
  and D.~Hjelm.
\newblock Mutual information neural estimation.
\newblock In {\em Proceedings of the 35th International Conference on Machine
  Learning}, Proceedings of Machine Learning Research, pages 531--540, 2018.

\bibitem{Bickel1967}
P.~Bickel.
\newblock Some contributions to theory of order statistics.
\newblock In {\em Proceedings of the 5th Berkeley Symposium on Mathematical
  Statistics and Probability, vol. 1}, pages 575--591, 1967.

\bibitem{BickelDoksum}
P.~Bickel and K.~Doksum.
\newblock {\em Mathematical Statistics}.
\newblock CRC, London, 2015.

\bibitem{Breiman2001}
L.~Breiman.
\newblock Random forests.
\newblock {\em Machine Learning}, 45(1):5–32, 2001.

\bibitem{Dua2019}
Kelly Markelle, Longjohn Rachel, and Nottingham Kolby.
\newblock {UCI} {M}achine {L}earning {R}epository, 2023.

\bibitem{PyTorch19}
A.~Paszke, S.~Gross, F.~Massa, A.~Lerer, J.~Bradbury, G.~Chanan, T.~Killeen,
  Z.~Lin, N.~Gimelshein, L.~Antiga, A.~Desmaison, A.~Kopf, E.~Yang, Z.~DeVito,
  M.~Raison, A.~Tejani, S.~Chilamkurthy, B.~Steiner, L.~Fang, J.~Bai, and
  S.~Chintala.
\newblock Pytorch: An imperative style, high-performance deep learning library.
\newblock In {\em Advances in Neural Information Processing Systems}, NIPS'19,
  pages 8024--8035, 2019.

\end{thebibliography}
\end{document}